\documentclass[11pt,usenames,dvipsnames]{article}
\usepackage{enumerate}
\usepackage{amsmath,amssymb}
\usepackage{natbib}
\usepackage{caption}
\usepackage{enumitem}
\usepackage{comment}
\usepackage[usenames]{color}
\usepackage{bm}
\usepackage{ying}
\usepackage{multirow}
\usepackage{rotating}
\usepackage{fullpage}
\usepackage{authblk}
\usepackage{mathtools}
\usepackage{geometry}
\usepackage{parskip}
\usepackage{float}
\usepackage{subcaption}
\usepackage{tikz,tabularx}
\usetikzlibrary{patterns}
\usetikzlibrary{arrows}
\usetikzlibrary{tikzmark, positioning, fit, shapes.misc,backgrounds}
\usetikzlibrary{decorations.pathreplacing, calc}
\tikzset{brace/.style={decorate, decoration={brace}},
  brace mirrored/.style={decorate, decoration={brace,mirror}},
}

\newcolumntype{g}{>{\columncolor{red}}c}

\usepackage{graphicx,amssymb}
\usepackage{xcolor}

\usepackage[colorlinks,
            linkcolor=red,
            anchorcolor=blue,
            citecolor=blue
            ]{hyperref}
\usepackage{algorithm}
\usepackage{algorithmic}

\let\emptyset\varnothing

\def \calib {\textrm{calib}}
\def \train {\textrm{train}}

\def \test {\textrm{test}}

\def \quant {\textrm{Quantile}}

\def \RA {{\textnormal{RA}}}
\def \uquant {\textnormal{Quantile}^+}
\def \learn {\textnormal{learn}}

\def \low {\textnormal{low}}
\def \high {\textnormal{high}}
\def \OPT {\textnormal{OPT}}

\providecommand{\keywords}[1]
{{ 
  \fontsize{9}{12}\selectfont
  \textbf{\textit{Keywords:}} #1
}}

\theoremstyle{plain}

\allowdisplaybreaks

\usepackage{hyperref}
\def\@#1\@{\begin{align}#1\end{align}}
\def\$#1\${\begin{align*}#1\end{align*}}

\usepackage{color-edits}
\addauthor[Ying]{ying}{magenta}
\addauthor[Yurui]{yurui}{blue}

\usepackage{xcolor}
 
\title{Prediction Sets for Counterfactual Decisions: \\ Coverage, Optimality, and Conformal Prediction}
\author[1]{Yurui Zheng}
\author[2]{Ying Jin} 
\affil[1]{School of Mathematical Sciences, Peking University}
\affil[2]{Department of Statistics and Data Science, University of Pennsylvania}
\date{}

\begin{document}

\maketitle

\begin{abstract}
Predictions are increasingly used to guide high-stakes decisions, from treatment selection to policy making. To ensure reliability with imperfect predictions, uncertainty quantification methods such as conformal prediction build prediction sets with coverage guarantees. However, statistical validity alone does not immediately determine the decisions to take,  nor the optimality thereof. This gap is especially delicate in counterfactual settings where the outcome that materializes depends on the action taken, so uncertainty cannot be specified independently of the decision rule. 

We develop a decision-theoretic framework for uncertainty-informed counterfactual decisions. 
We identify a novel notion of \emph{policy-coupled coverage}---namely, coverage of the realized outcome under the action induced by the prediction sets themselves---as the optimal and lossless interface between uncertainty and action. 
It plays three roles. First, it justifies acting via a natural max--min rule as minimax-optimal under distributional ambiguity. Second, optimizing prediction sets under policy-coupled coverage is equivalent both to a stronger universal-coverage formulation and to the direct risk-averse optimization over policies and utility certificates; this equivalence yields the explicit form of the population-optimal prediction sets. Third, it admits a two-stage procedure, Policy-Coupled Risk-Averse Conformal Prediction (PC-RACP), that approximates these optimal sets with rigorous finite-sample coverage. 
Simulations and a real email-marketing experiment confirm that PC-RACP delivers higher utility than existing approaches while maintaining valid coverage, and that ignoring the counterfactual structure of the decision problem is suboptimal for both validity and utility.\footnote{Reproduction code for the experiments can be found in \url{https://github.com/yurui-zheng/PC-RACP}.}

\end{abstract}


\section{Introduction}

\subsection{Prediction, uncertainty, and counterfactual decisions}

Predictions are often used to guide decisions whose consequences will only be realized after they are made. A physician choosing among treatments, an online platform selecting product recommendations, or a policymaker allocating interventions may rely on  predictions to probe the likely outcomes and guide their actions~\citep{manski2004statistical,li2010contextual,athey2015machine,gao2024causal,feuerriegel2024causal}.   
Since predictions are inherently imperfect, a critical challenge in high-stakes problems is how to act reliably under predictive uncertainty. 
A natural starting point is to quantify such uncertainty, 
for example by building prediction sets---ranges of values the unknown outcome may take---around the predictions. To ensure reliability, an immediate  goal is \emph{statistical validity}. In particular, conformal prediction offers finite-sample coverage guarantees under minimal assumptions, making it an attractive uncertainty quantification framework~\citep{vovk2005algorithmic}. 

However, statistical validity alone does not fully address the role of uncertainty quantification in supporting decision-making. Many existing methods  do not explain how to derive downstream decisions based on estimated uncertainty.  In general, it is not straightforward how validity (e.g., coverage)  may translate to guarantees for downstream actions.  
If the ultimate goal is to make ``better'' uncertainty-informed decisions, the fundamental question is decision-theoretic: what decisions do uncertainty guarantees justify, what optimality criterion do they correspond to, and when does  uncertainty quantification preserve the right amount of information needed for optimal action?

These questions have motivated recent work on decision-theoretic foundations of predictive inference in standard prediction problems, where there is a \emph{single} realized outcome $Y\in \cY$ to be predicted~\citep{kiyani2025decision,wang2026optimal}. In that setting, prediction sets $C(X)\subseteq \cY$ obeying the usual \emph{marginal coverage} $\Pr(Y\in C(X))\ge 1-\alpha$ can serve as a lossless interface between uncertainty and decision: for risk-averse agents that maximize a high-probability bound of realized utility, the policy and utility certificate induced by optimally designed prediction sets attain the same value as those obtained by directly optimizing over policies and utility certificates. This picture provides essential insights on the value of prediction sets for decision-making, but a key premise is that the object to be covered is a fixed outcome \(Y\) independent of the action  taken.

\vspace{-0.75em}

\paragraph{Counterfactual decision-making.} Many decision-making problems, however, are \emph{counterfactual}: the outcome that will be realized depends on the action taken. Such a structure is standard in many tasks spanning clinical trials, marketing, and policy-making, captured by the potential-outcome framework in causal inference and contextual bandits~\citep{rubin2005causal,imbens2015causal}.
Let $\cA=\{a_1,\dots,a_{|\cA|}\}$ be a finite set of actions. Each action $a\in\cA$ induces a distinct
potential outcome $Y(a)\in\cY$, and the potential outcomes $\{Y(a)\}_{a\in \cA}$ are treated as separate random variables. With the
chosen action $A$, the realized outcome is $Y=Y(A)$. 
Accordingly,  the realized utility is $u(A,Y(A))$, where the utility function $u\colon \cA\times\cY\to \RR$ describes the value $u(a,y)$ of taking action $a\in \cA$ and observing outcome $y\in \cY$. The  
setting of~\cite{kiyani2025decision} is the special case where $Y(a)\equiv Y$ for all $a\in \cA$. 

In the counterfactual setting, uncertainty can no longer be specified independently of the decision rule since the latter changes the realized outcome. 
It is therefore unclear what coverage guarantee to begin with, and whether, and how, uncertainty quantification with such coverage connects to optimal counterfactual decisions. 
This raises the central question of this paper: 

\vspace{0.25em}
\begin{quote}
    \emph{What is the right interface between uncertainty quantification and counterfactual decision-making where actions determine which potential outcome is realized?}
\end{quote}
\vspace{0.25em}

Several natural validity notions may come to mind. One could require separate (marginal) coverage of every potential outcome,  coverage of the realized outcome under every possible policy, or only coverage of the outcome realized under the policy induced by the prediction sets.  

Our results reveal the role of \emph{policy-coupled coverage} (formalized in Section~\ref{sec:prelim}) as the decision-theoretically optimal interface. Figure~\ref{fig:intro} summarizes both the practical deployment pipeline and the theoretical organization of our framework. The top row shows that our practical procedure \emph{Policy-Coupled Risk-Averse Conformal Prediction} (PC-RACP) produces a collection of prediction sets satisfying the coverage guarantee, which further induce actions based on the prediction sets. This pipeline is justified by our theoretical components in the bottom row of Figure~\ref{fig:intro}, which we preview in the next and summarize in Figure~\ref{fig:vis}.

\begin{figure}
    \centering
    \includegraphics[width=0.9\linewidth]{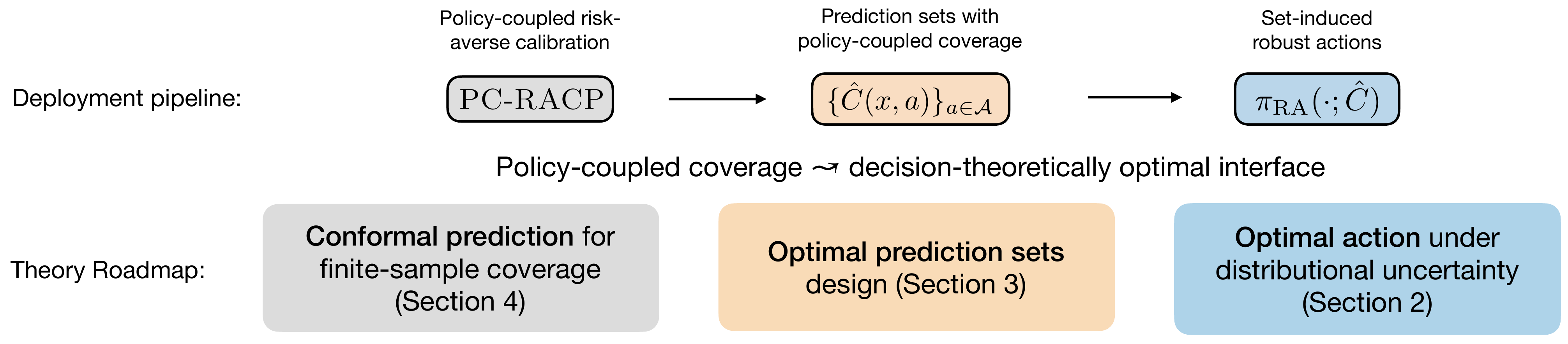}
    \caption{{\small Top: deployment pipeline. PC-RACP constructs action-indexed prediction sets $\{\widehat C(x,a)\}_{a\in\cA}$ with policy-coupled coverage, and acting on these sets via the counterfactual max--min rule yields set-induced risk-averse decisions. Bottom: theory roadmap. Section~\ref{sec:prelim} characterizes the fixed-set robust decision rule under distributional ambiguity, Section~\ref{sec:opt} studies optimal prediction-set design, and Section~\ref{sec:alg} gives the conformal procedure for constructing (nearly) optimal prediction sets with finite-sample coverage.}}
    \label{fig:intro}
\end{figure}
 
\subsection{Preview of results}

The first challenge in the counterfactual setting is that there is no single unknown outcome to be ``covered'' in advance. 
For a decision rule $\pi\colon \cX\to \cA$ mapping from the feature space to the actions, the realized outcome $Y(\pi(X))$ is naturally of primary interest, yet its distribution varies with $\pi$.  Consequently, the first question is what uncertainty object is being constructed and what coverage guarantee it should satisfy.

A natural idea is to consider a collection of prediction sets $\{C(x,a)\}_{a\in\cA}$, one per action's potential outcome. 
Our first result identifies the validity notion and decision rule that match risk-averse counterfactual decision-making for a fixed collection of  prediction sets. Somewhat surprisingly, rather than the widely-studied per-outcome marginal coverage~\citep{lei2021conformal,jin2023sensitivity}, a \emph{policy-coupled coverage} (Section~\ref{subsec:coverage}) over the realized outcome under the max--min rule 
\[
\pi_{\mathrm{RA}}(x;C)=\argmax_{a\in\cA}\inf_{y\in C(x,a)}u(a,y) 
\]
achieves the optimal performance under distributional uncertainty. 
Specifically, for the risk-averse objective  in~\cite{kiyani2025decision},  the policy $\pi_{\mathrm{RA}}(\cdot;C)$ maximizes the worst-case objective among all data distributions inducing the policy-coupled coverage of $\{C(x,a)\}_{a\in\cA}$. Thus, policy-coupled coverage is the decision-relevant notion of validity, and $\pi_{\mathrm{RA}}(\cdot;C)$ is the optimal decision rule.

Upon establishing the match between coverage and utility objective for a \emph{given} set of prediction sets, we  study optimal design of prediction sets (Section~\ref{sec:opt}). 
Our second result reveals that optimizing prediction sets under policy-coupled coverage attains the same utility objective as directly optimizing over policies and utility certificates (Theorem~\ref{thm:equi-dpo}). This establishes prediction sets as a lossless interface for risk-averse counterfactual decision-making, extending the conclusions in~\cite{kiyani2025decision}. 

Our third result shows that policy-coupled coverage is equivalent in the decision-theoretic sense to \emph{universal coverage}, i.e., requiring valid coverage for the realized outcome under any policy, although the former is strictly weaker (Theorem~\ref{thm:equi-cpo}). Such an equivalence enables us to rely on either of them to derive the formulae of the optimal prediction sets and construct prediction sets with finite-sample validity. The universal coverage proves useful in the first task (Theorem~\ref{thm:unified}). 

Finally, we return to the construction of prediction sets obeying the policy-coupled coverage with data (Section~\ref{sec:alg}). We propose a weighted two-stage conformal prediction procedure  and prove the coverage guarantee of the resulting prediction sets, i.e., they cover  the realized outcome under the induced max--min rule  with high probability. In simulation studies and real data case studies in Section~\ref{sec:simulation}, we demonstrate the high utility achieved by the max--min policy based on, and the valid coverage of, our prediction sets. In particular, ignoring the counterfactual structure and following the single-outcome recipe is suboptimal in terms of both statistical validity and decision utility.

\begin{figure} 
    \centering
    \includegraphics[width=\linewidth]{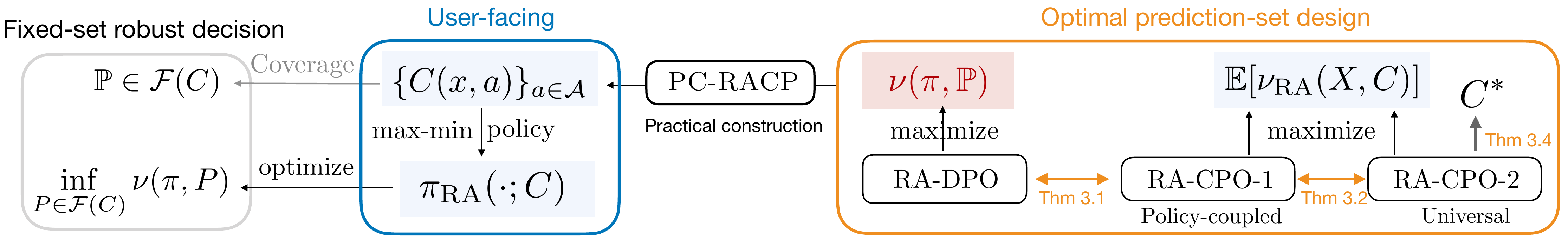}
    \caption{{\small  
Blue box: the user-facing prediction sets $\{C(x,a)\}_{a\in\cA}$ and  the induced max--min policy $\pi_{\mathrm{RA}}(\cdot;C)$. 
Grey box (Sec.~\ref{sec:prelim}): Under policy-coupled coverage, the prediction sets induce an ambiguity class $\cF(C)$, for which $\pi_{\mathrm{RA}}$ is worst-case optimal for risk objective $\nu(\pi,P)$. 
Orange box (Sec.~\ref{sec:opt}): Optimizing prediction sets under policy-coupled coverage (RA-CPO-1) is equivalent both to the stronger universal-coverage formulation (RA-CPO-2) and to direct risk-averse optimization over policies and utility certificates (RA-DPO). 
Finally, PC-RACP produces conformal prediction sets obeying the policy-coupled coverage in finite samples.}}
    \label{fig:vis}
\end{figure}

\subsection{Related work}

Our work builds on recent decision-theoretic foundations for predictive inference~\citep{kiyani2025decision,wang2026optimal,zhu2026conformalriskaversedecisionmaking}. This literature complements the classical literature on calibration, where probabilistic forecasts are shown to be optimal for risk-neutral decision-making~\citep{foster1998asymptotic,kakade2008deterministic,zhao2021calibrating}. In contrast, our work follows the risk-averse framework of~\cite{kiyani2025decision} to establish optimal decision-making based on prediction sets.
\citet{kiyani2025decision} show that, in standard single-outcome prediction problems, prediction sets can serve as a lossless interface between uncertainty and risk-averse decision-making; \citet{wang2026optimal} extend this perspective to expected-loss objectives. The distinguishing feature of the counterfactual setting is that the decision changes the outcome to be covered, which leads to distinct coverage target, decision-theoretic optimality, and construction of prediction sets. The recent concurrent work of~\cite{zhu2026conformalriskaversedecisionmaking} considers action-conditional coverage in the single-outcome setting; although they tie the action closely to the coverage, due to the counterfactual setting, we build distinct optimality theory, and  we construct a set of per-action prediction sets instead of one for the single realized outcome.

Our work is also related to conformal inference for counterfactual and causal quantities under the potential outcome framework in causal inference and data-driven decision-making~\citep{rubin2005causal,imbens2015causal}. This framework is standard in our motivating applications such as treatment selection and online marketing. 
Prior work develops conformal prediction sets for potential outcomes, individual treatment effects, and sensitivity analysis~\citep{lei2021conformal,jin2023sensitivity,yin2024conformal,alaa2023conformal}, whose validity is often  per-action or per-potential-outcome coverage guarantees. 
By contrast, our goal is to connect statistical validity to downstream decision optimality. 
This leads to the new notion of policy-coupled coverage of the realized outcome under the policy induced by the prediction sets, with different motivations and conformal procedures.


\section{A decision-theoretic framework for counterfactual decisions}
\label{sec:prelim}

We begin by introducing the problem setup and the decision-theoretic framework. 
Section~\ref{subsec:objective} introduces the utility objective studied in this paper.  Section~\ref{subsec:set_action} introduces the notion of prediction sets and set-based decisions. Section~\ref{subsec:coverage} studies the fixed-set optimal decision problem (grey box in Figure~\ref{fig:vis}) which justifies acting based on prediction sets under distributional ambiguity. 

Recall that $X\in\mathcal{X}$ is the observed features, $\mathcal{A}=\{a_1,\ldots,a_{|\cA|}\}$ a finite action set, and
$\{Y(a)\}_{a\in\mathcal{A}}$ potential outcomes, where $Y(a)\in\mathcal{Y}$ is the outcome realized if
action $a$ is taken. A bounded utility function $u:\mathcal{A}\times\mathcal{Y}\to\mathbb{R}$ assigns value $u(a,y)$ to taking action $a$ and observing outcome $y$. A policy $\pi:\mathcal{X}\to\mathcal{A}$ maps features to actions; under policy $\pi$, the realized outcome is $Y(\pi(X))$ and the realized utility is $u(\pi(X),Y(\pi(X)))$. The defining feature of our setting is that \emph{the outcome that materializes depends on the action taken}. 
We use the notation $P$ to denote a distribution over $(X,\{Y(a)\}_{a\in \cA})$, and for clarity, let $\PP$ denote the unknown, true data distribution.

\subsection{Risk-averse objective: the population-level problem}
\label{subsec:objective}
  
The decision-making problem aims to choose actions that optimize certain objectives.
Following~\cite{kiyani2025decision}, we study  risk-averse agents whose goal is
to choose actions that ensure a sufficiently high utility with
high probability over the randomness of unknown outcomes.
Formally, fix some $\alpha\in(0,1)$ and let $u_{\max}\ge \max_{a\in \cA}\sup_{y\in\cY}u(a,y)$ be a fixed upper bound on the utility. We call $\eta\colon\cX\to(-\infty,u_{\max}]$ a \emph{utility certificate} for $\pi$ under distribution $P$ if
\@\label{eq:cert}
P\bigl(u(\pi(X),Y(\pi(X)))\ge\eta(X)\bigr)\ge 1-\alpha.
\@
Here,~\eqref{eq:cert} extends the utility certificate notion in~\cite{kiyani2025decision}  by replacing the standard label $Y$ with the realized outcome $Y(\pi(X))$. 
The feature-dependent certificate thus concerns the realized utility: with probability at least $1-\alpha$, policy $\pi$ delivers at least $\eta(X)$ of utility. 
The larger $\eta(X)$ is, the better the policy serves a risk-averse agent. 
This motivates the risk objective
\@\label{eq:nu}
\nu(\pi,P):=\max\Bigl\{\EE_P[\eta(X)]\colon
P\bigl(u(\pi(X),Y(\pi(X)))\ge\eta(X)\bigr)\ge 1-\alpha\Bigr\},
\@
the highest expected certificate attainable under policy $\pi$ and distribution $P$. Such a risk objective is justified (see~\citet[Section 2]{kiyani2025decision}) as a marginalized version of an $X$-conditional high-probability risk optimization problem. 
The marginalization is useful since conditional coverage is often intractable in finite samples~\citep{foygel2021limits}. 

The direct risk-averse optimization problem maximizes $\nu(\pi,\PP)$ over $\pi\in\Pi$:
\begin{equation}
\tag{RA-DPO}\label{eq:def_ra_dpo}
\begin{aligned}
   \mathop{\text{Maximize}}_{\pi\in \Pi}~~\nu(\pi,\PP) \quad \Leftrightarrow \quad  \mathop{\text{Maximize}}_{\pi\in \Pi,~\eta(\cdot)} ~~&\EE[\eta(X)] \\ 
\text{subject to} ~~ & \PP\big[u(\pi(X),Y(\pi(X))) \geq \eta (X)\big] \geq 1-\alpha.
\end{aligned}
\end{equation} 
Here $\Pi$ is the collection of all policies $\pi\colon \cX\to \cA$.
  
\subsection{Prediction sets and set-based actions}
\label{subsec:set_action}

Since the true joint distribution $\PP$ of $(X,\{Y(a)\}_{a\in\mathcal{A}})$ is typically unknown, the decision-maker needs to rely on partial information about $\PP$. We focus on the scenario where such partial information is encapsulated in a collection of prediction sets $C=\{C(x,a)\}_{a\in\mathcal{A}}$. Intuitively, for each  action $a\in \cA$, the prediction set $C(x,a)\subseteq\mathcal{Y}$ quantifies the likely range of the outcome $Y(a)$ when observing feature values $X=x\in \cX$. 

Without any formal justification of how to act upon the prediction sets,  a natural instinct is to assume $Y(a)$ varies within $C(x,a)\subseteq\mathcal{Y}$ and, extending~\cite{kiyani2025decision}, take the action maximizing worst-case utility within its set,
\@\label{eq:pi_ra}
\pi_\RA(x;C):=\argmax_{a\in\cA}\inf_{y\in C(x,a)}u(a,y),\qquad
\nu_\RA(x;C):=\max_{a\in\cA}\inf_{y\in C(x,a)}u(a,y).
\@
While it is common in practice to enforce non-empty prediction sets, to avoid additional assumptions on either the data-generating process or the utility function, we set the convention  $\inf_{y \in \varnothing}u(a,y):=u_{\max}$. 
Intuitively, $\pi_\RA$ guards against the worst outcome $Y(a)$ within the likely range $C(x,a)$, representing a \emph{risk-averse} way of acting based on the prediction sets. 

At this stage, $\pi_{\mathrm{RA}}(\cdot;C)$ and $\nu_{\mathrm{RA}}(\cdot;C)$ are set-induced quantities only. 
The remaining question is decision-theoretic: what validity guarantee on $C$ makes this rule justified, and in what sense?

\subsection{What coverage guarantee justifies set-induced counterfactual decisions?}
\label{subsec:coverage}

Our conceptual basis is coverage-induced distributional ambiguity. A validity
guarantee on the prediction sets  defines a class of data-generating distributions $P$ consistent with it, and one chooses the policy that maximizes the worst-case risk-averse objective. 
The question of this section is therefore: which counterfactual validity notion
makes a set-induced policy robust-optimal?

Unlike the standard prediction problem, the coverage notion for counterfactual decisions is not straightforward since the realized outcome is coupled with the action to be determined. 
A direct extension of the standard marginal coverage is \emph{per-action marginal coverage}:
\@\label{eq:per_action}
P(Y(a)\in C(X,a))\ge 1-\alpha\quad\text{for each }a\in\cA,
\@
requiring each set to individually cover its potential outcome with probability $1-\alpha$. This has been the primary focus in predictive inference in counterfactual (causal) problems~\citep{lei2021conformal,jin2023sensitivity,yin2024conformal}. 

However, the per-action coverage which constrains each set in isolation appears insufficient for decision-making: it does not by itself determine whether the set for the \emph{chosen} action covers the \emph{realized} outcome.  That is,~\eqref{eq:per_action} alone does not imply whether, for a policy $\pi$ under consideration, 
\@\label{def:cov_pi}
\PP\big(  Y(\pi(X))\in C(X,\pi(X))\big)\geq 1-\alpha.
\@
It is therefore unclear how acting based on prediction sets satisfying~\eqref{eq:per_action} may be justified. 
Furthermore, even though a ``universal'' coverage holds---that is,~\eqref{def:cov_pi} holds for all policy $\pi$---the decision-maker may have to compare a collection of \emph{stochastically dependent} prediction sets $\{C(X,\pi(X))\}_{\pi\in\Pi}$, which may not necessarily lead to coverage for the \emph{chosen} decision-induced outcome. 

We address this question by identifying the notion of coverage that justifies $\pi_{\RA}(\cdot;C)$. 
We define the distributional uncertainty set (omitting the dependence on $C$ in $\pi_{\RA}$ for readability)
\@\label{eq:def_P_set}
\cF(C) = \Big\{P=\{P_a(x,y)\}_{a\in \cA}\colon P\big(Y(\pi_\RA(X)) \in C(X,\pi_\RA(X)) \big) \ge 1-\alpha \Big\},
\@
and call the corresponding coverage requirement as \emph{policy-coupled coverage} (PCC for short). It requires the prediction sets corresponding to the action $\pi_{\RA}(X)$ to cover the realized outcome with high probability. This guarantee is decision-aware---and self-referential---in that the sets must cover the realized outcome induced by them. To the best of our knowledge, existing finite-sample methods do not provide this guarantee; we return to this construction problem in Section~\ref{sec:alg}.

Theorem~\ref{thm:op-decision} connects the coverage notion, max--min decision rule, and utility objective in the decision-theoretic sense. Its proof is in Appendix~\ref{app:subsec_op_decision}.
\begin{theorem}\label{thm:op-decision}
  Let $\cF(C)$ be defined in~\eqref{eq:def_P_set}. Then  
  \$
  \argmax_{\pi \in \Pi} \inf_{P \in \cF(C)} \nu(\pi,P)= \pi_\RA(\cdot;C)
  \$
\end{theorem} 

Theorem~\ref{thm:op-decision} shows that, if the only information available about the underlying distribution is that it belongs to \(\mathcal F(C)\), then  \(\pi_{\mathrm{RA}}(\cdot;C)\) maximizes the worst-case
risk-averse objective over that ambiguity class. In this sense, policy-coupled
coverage makes the prediction sets sufficient for deriving the minimax-optimal decision rule for a \emph{given} collection of prediction sets.

This result justifies the following pipeline: for a new instance $X$ with prediction sets $\{C(X,a)\}_{a\in \cA}$ obeying the PCC, the decision-maker chooses the action $\pi_{\RA}(X)=\pi_{\RA}(X;C)$;  
due to the PCC,  
\@\label{eq:nu_certificate}
\PP\big( u(\pi_{\RA}(X); Y(\pi_{\RA}(X))) \geq \nu_{\RA}(X;C)\big) \geq 1-\alpha.
\@
That is, $\nu_{\RA}(X;C)$ is now a valid \emph{utility certificate} (c.f.~definition in~\eqref{eq:cert}) for the policy $\pi_{\RA}(\cdot;C)$.

To summarize, this section resolves the fixed-set decision-making question: once a collection of prediction sets is given, policy-coupled validity justifies acting via \(\pi_{\mathrm{RA}}\). The remaining question is the set-level optimality:  Which collections \(\{C(x,a)\}_{a\in\cA}\) are most useful to a decision-maker among all sets satisfying the PCC? Can they match the value of direct policy-certificate optimization?


\section{Optimal prediction sets for counterfactual decisions}
\label{sec:opt}

This section completes the theoretical framework by studying the questions above.  
We shall show that prediction sets are a lossless interface for counterfactual decision-making: optimizing $\EE[\nu_{\RA}(X;C)]$ with respect to $C$ attains the same optimal value as direct risk-averse optimization over policies and utility certificates. We then leverage an equivalent form of the first optimization problem to establish the explicit form of the optimizer $C^*$ which guides practical construction later on.

\subsection{Prediction sets as lossless interface: RA-DPO and two versions of RA-CPO}
\label{subsec:opt_equiv}

Before studying the prediction sets that are most useful to the decision-maker, let us recall the~\textnormal{\ref{eq:def_ra_dpo}} problem in Section~\ref{sec:prelim}. RA-DPO optimizes directly over policies and utility certificates. We show that optimizing prediction sets under policy-coupled coverage is lossless for this objective: optimally designed prediction sets attain the same objective value as RA-DPO.

Consider the following prediction-set optimization problem, named after~\cite{kiyani2025decision}: 
\begin{equation}
\label{eq:def_ra_cpo_1}\tag{RA-CPO-1}
\begin{aligned}
\mathop{\text{Maximize}}_{C(\cdot,\cdot)}~~ & \EE[\nu_\RA(X;C)] \\
\text{subject to} ~~ & \PP(Y(\pi_\RA(X;C)) \in C(X,\pi_\RA(X;C)) ) \ge 1-\alpha .
\end{aligned}
\end{equation}
RA-CPO-1 seeks to maximize the expected utility certificate for the policy $\pi_{\RA}$ induced by the prediction sets subject to the PCC. 
Compared with RA-DPO which optimizes over all policies and their associated utility certificates, 
a key distinction is that RA-CPO-1 optimizes over a seemingly smaller range: the policies and utility certificates considered must be induced by prediction sets obeying the PCC. In words, RA-CPO-1 uses prediction sets as an interface for risk-averse counterfactual decisions.
Somewhat surprisingly, the prediction-set setup is equivalent to the full-range optimization in RA-DPO. Theorem~\ref{thm:equi-dpo} formalizes the idea, whose proof is in Appendix~\ref{app:subsec_dpo_cpo}. 

\begin{theorem}[Equivalence of \textnormal{\ref{eq:def_ra_cpo_1}} and~\textnormal{\ref{eq:def_ra_dpo}}]\label{thm:equi-dpo}
  From any optimal solution $(\pi^*,\nu^*)$ to RA-DPO, there exists an optimal solution $\{C^*_1(X,a)\}_{a\in \cA}$ to RA-CPO-1  so that $\EE[\nu^*(X)]=\EE[\nu_\RA(X;C^*_1)]$. Vice versa, for any optimal solution $\{C^*_1(X,a)\}_{a\in \cA}$ to RA-CPO-1, there exists  an optimal solution $(\pi^*,\nu^*)$ to RA-DPO, such that $\EE[\nu^*(X)]=\EE[\nu_\RA(X;C_1^*)]$.
\end{theorem}

Theorem~\ref{thm:equi-dpo} establishes prediction sets with PCC as a lossless interface for counterfactual decision-making. 
For a decision-maker who is interested in maximizing the risk objective $\nu(\pi;\PP)$ over all possible policies, it suffices to consider the induced policy of an optimal prediction set with PCC. Formally, for any optimal solution $\{C_1^*(x,a)\}_{a\in \cA}$, by our discussion around~\eqref{eq:nu_certificate},  the induced policy $\pi_{\RA}(\cdot;C_1^*)$ has a valid certificate $\nu_\RA(X;C^*_1)$. Therefore, by definition~\eqref{eq:nu}, 
\$
\nu(\pi_{\RA}(\cdot;C_1^*);\PP)\geq \EE[\nu_\RA(X;C^*_1)] = \EE[\nu^*(X)],
\$
the optimal objective of RA-DPO. As such, $\pi_{\RA}(\cdot;C_1^*)$ must be an optimal solution to RA-DPO. 
In other words, finding the optimally designed prediction sets suffices for direct risk-averse optimization. This is the focus of Section~\ref{sec:opt_form}.

While RA-CPO-1 provides a sufficient summary, the self-referential nature of the PCC makes it less straightforward to explicitly find the optimal solution. 
It turns out helpful to consider another RA-CPO-type problem, which we call RA-CPO-2: 
\begin{equation}
\tag{RA-CPO-2}\label{eq:def_ra_cpo_2}
\begin{aligned}
\mathop{\text{Maximize}}_{C(\cdot,\cdot)}~~ & \EE[\nu_\RA(X;C)]\\
\text{subject to} ~~ & \PP(Y(\pi(X)) \in C(X,\pi(X)) ) \ge 1-\alpha ~~\text{for any } \pi \in \Pi.
\end{aligned}
\end{equation}
The two RA-CPO-type problems pursue the same objective function, yet RA-CPO-2 imposes a stronger coverage condition: the prediction sets need to provide ``universal'' policy-induced coverage for any policy $\pi\in\Pi$, while RA-CPO-1 does so for $\pi_{\RA}$ only.  
However, RA-CPO-1 and RA-CPO-2 are equivalent for the risk objective.  The proof of Theorem~\ref{thm:equi-cpo} is in Appendix~\ref{app:subsec_equiv_cpo}. 

\begin{theorem}[Equivalence of \textnormal{\ref{eq:def_ra_cpo_1}} and~\textnormal{\ref{eq:def_ra_cpo_2}}] \label{thm:equi-cpo} 
  For any optimal solution $\{C^*_1(X,a)\}_{a\in \cA}$ of RA-CPO-1, 
  there exists an optimal solution $\{C^*_2(X,a)\}_{a\in \cA}$ to RA-CPO-2 so that $\EE[\nu_\RA(X,C_1^*)]=\EE[\nu_\RA(X,C_2^*)]$. 
  Moreover, any optimal solution $\{C^*_2(X,a)\}_{a\in \cA}$ of RA-CPO-2 must also be an optimal solution to RA-CPO-1.  
\end{theorem}

The equivalence of RA-CPO-1 and RA-CPO-2 illustrates that $\pi_{\RA}(x;C)$ is the policy that ``scrutinize'' the coverage ability of $C$ the most from a risk-averse decision-making perspective.  

Together, Theorems~\ref{thm:equi-dpo} and~\ref{thm:equi-cpo} show that  RA-DPO, RA-CPO-1 and RA-CPO-2 are all equivalent. 
Following the discussion below Theorem~\ref{thm:equi-dpo}, any optimal solution $\{C_2^*(x,a)\}$ to RA-CPO-2 must also be optimal for RA-CPO-1, which is further an optimal solution to RA-DPO. 
In words, to solve RA-DPO, it suffices to find the optimal prediction sets that maximize $\EE[\nu_{\RA}(X;C)]$ while obeying the coverage guarantees imposed in RA-CPO-1 or RA-CPO-2. 
We shall see that RA-CPO-2 is convenient for deriving the explicit form of the optimal solutions (Section~\ref{sec:opt_form}), while the coverage constraint of RA-CPO-1 is arguably easier to achieve in finite samples (Section~\ref{sec:alg}).

\paragraph{Summary of theoretical results.} We have thus far completed the decision-theoretical framework of risk-averse counterfactual decisions, summarized in Figure~\ref{fig:vis}. Our results reveal the key role of policy-coupled coverage: distinct from \cite{kiyani2025decision} where the usual marginal coverage $\PP(Y\in  {C}(X))\geq 1-\alpha$ suffices, here, we need stronger coverage guarantees for the prediction sets to be optimal in guiding decision making. The policy-coupled coverage justifies acting based on prediction sets. In addition, optimally designed prediction sets with such coverage attain the same optimal value as direct risk-averse optimization for a risk-averse decision-maker.

\subsection{Optimal prediction sets for risk-averse decisions}
\label{sec:opt_form}

Finally, we provide an explicit form of the (population-level) optimal prediction sets   by solving the~\ref{eq:def_ra_cpo_2} problem. 

As preparation for solving RA-CPO-2, we define some key quantities.
For a random variable $Z$, we define its ``upper quantile'' as 
\@\label{eq:def_upp_qt}
\uquant_\alpha[Z]:= \sup \{z \in \RR \given \PP(Z\le z) \le \alpha\},\qquad\alpha \in [0,1),
\@
and $\uquant_1[Z]:= u_{\max}$. This definition slightly differs from the usual one, i.e.,  $\quant_\alpha[Z] := \inf\{z \in \RR \given \PP(Z\leq z) \geq \alpha\}$, and one can prove that $\uquant_\alpha[Z] \ge \quant_\alpha[Z]$. 
The right-continuity of the function $q(\alpha)= \sup \{z \in \RR \given \PP(Z\le z) \le \alpha\}$ makes it convenient to define optimal prediction sets; see a discussion in Appendix~\ref{app:subsec_right_continuity}. 
In addition, we define the following quantities:
\$
&\gamma(x,t,a)=\uquant_{1-t}[u(a,Y(a)) \given X=x],\\
&\theta(x,t)=\max_{a \in \cA}\,\uquant_{1-t}[u(a,Y(a)) \given X=x],\\
&a(x,t)=\argmax_{a \in \cA} \,\uquant_{1-t}[u(a,Y(a)) \given X=x].
\$
Intuitively, $\gamma(x,t,a)$ is the upper conditional quantile of the utility when taking the action $a\in \cA$, and $\theta(x,t)$ is the optimal quantile over actions, which is achieved by the action $a(x,t)$.

Proposition~\ref{thm:optimal-set}, which considers a conditional problem, serves as an instrument for defining the optimal prediction sets. 
A more general result is in Proposition~\ref{thm:random-optimal-set} with proof in Appendix~\ref{app:subsec_optimal_set}.  

\begin{prop}\label{thm:optimal-set}
  For any fixed $x\in \cX$ and  $t \in [0,1]$, among all the prediction sets $\{C(x,a)\}_{a\in \cA}$ obeying $\min_{a\in \cA} \PP(Y(a)\in C(x,a) \given X=x) \ge t$, the following prediction sets maximize   $\nu_\RA(x;C)$:
  \$
  C^*(x,a)=\{y \in \cY \given u(a,y) \ge \gamma(x,t,a)\} ~~\text{for every action}~~ a \in \cA.
  \$
  Further, we have $\nu_\RA(x;C^*)=\theta(x,t)$.
\end{prop}

We now proceed to solve~\ref{eq:def_ra_cpo_2}. For any feasible prediction sets 
$C=\{C(x,a)\}_{a\in\cA}$ in RA-CPO-2, define
\$
t_C(x)=\min_{a\in \cA}\PP(Y(a)\in C(x,a)\given X=x).
\$
Then $t_C$ is feasible for~\ref{eq:def_ra_cpo_2'} below, and the objective value of $C$ is no larger than $\EE_X[\theta(X,t_C(X))]$. Conversely, for any feasible $t:\cX\to[0,1]$, Proposition~\ref{thm:optimal-set} constructs threshold prediction sets
\$
C_t(x,a)=\{y\in\cY\given u(a,y)\ge \gamma(x,t(x),a)\},\qquad a\in\cA,
\$
which are feasible for RA-CPO-2 and attain the same objective value $\EE_X[\theta(X,t(X))]$. RA-CPO-2 is thus equivalent to the following program:
\begin{equation}
\tag{RA-CPO-2'}\label{eq:def_ra_cpo_2'}
\begin{aligned}
\mathop{\text{Maximize}}_{t:\cX \to [0,1]} ~~&\EE_X[\theta(X,t(X))]\\
\text{subject to} ~~& \EE_X[t(X)] \ge 1-\alpha.
\end{aligned}
\end{equation}

In addition, the above arguments imply that for any optimal solution $t^*(x)$ to RA-CPO-2', we can find an optimal solution  to RA-CPO-2 given by 
\@\label{eq:opt}
C^*(x,a)=\{y \in \cY \given u(a,y) \ge \gamma(x,t^*(x),a)\} ~~\text{for every action}~~ a \in \cA.
\@
Here, the function $t(\cdot)$ can be viewed as a conditional coverage assignment rule: the constructed prediction sets have action-wise conditional coverage at least $t^*(x)$, and the constraint in RA-CPO-2' ensures the desired marginal universal coverage. Meanwhile, the optimal objective value of RA-CPO-2' is equal to that of RA-CPO-2, i.e., the expectation of the optimal utility certificate $\theta(X,t^*(X))$.

It then suffices to find the optimal solution $t^*(x)$ to RA-CPO-2'. Define
\@\label{eq:def_g}
g(x,\beta):= \argmax_{s \in [0,1]} ~\{\theta(x,s)+\beta s\}.
\@
Since $\theta(x,t)$ is left-continuous and non-increasing in $t$ (because $\gamma(x,t,a)$ is left-continuous in $t$; c.f.~Appendix~\ref{app:subsec_right_continuity}) and $\beta t$ is continuous in $t$, it is straightforward to see that $g(x,\beta)$ is well-defined.

Theorem~\ref{thm:unified} characterizes the optimal prediction sets, assuming that the maximizer of $\theta(x,s)+\beta s$ over $[0,1]$ is unique. A fully general result that eliminates such conditions via randomization is in Theorem~\ref{thm:seed_unified} with proof in Appendix~\ref{app:subsec_proof_unified_opt}. 

\begin{theorem}\label{thm:unified}
  Suppose the maximizer of $\theta(x,s)+\beta s$ is unique for $P_X$-a.e.~$x$ for all $\beta\geq 0$, then there exists some fixed constant $\beta^*\ge 0$ such that 
  \$
  t^*(x)=g(x,\beta^*)
  \$
  is an optimal solution to RA-CPO-2' and $\EE [g(X,\beta^*)]= 1-\alpha$. Accordingly, the prediction sets 
  \@\label{eq:def_opt_sets}
    C^*(x,a) = \big\{y \in \cY \colon  u(a,y) \ge \gamma(x,g(x,\beta^*),a) \big\} ~~\text{for every action}~~ a \in \cA
  \@
  are an optimal solution to RA-CPO-2 and RA-CPO-1.
\end{theorem}

\vspace{-0.5em}
\paragraph{Summary of optimal solutions.} The optimal prediction sets now connect all the concepts established before, yielding optimal worst-case utility for a risk-averse agent. 
By definition, we have $\inf_{y
\in C^*(x,a)} u(a,y)=\gamma(x,g(x,\beta^*),a)$ (see Appendix~\ref{app:subsec_cont_conv} for a formal proof).
The max--min policy based on $C^*$ is then $\pi^*_\RA(x)=\argmax_{a\in\cA}\gamma(x,g(x,\beta^*),a)$. For this policy, $\eta_{\RA}^*(x):=\max_{a\in \cA}\gamma(x,g(x,\beta^*),a)$ is a valid utility certificate because of the PCC of $C^*$. 
By Theorems~\ref{thm:equi-dpo} and~\ref{thm:equi-cpo}, we know $(\pi_{\RA}^*,\eta_{\RA}^*)$ is the optimal solution to RA-DPO, and $C^*$ is the optimal solution to RA-CPO-1 and RA-CPO-2.


\section{Constructing optimal conformal prediction sets}
\label{sec:alg}

In this section, we present an algorithm for practically constructing the prediction sets. 
Theorem~\ref{thm:unified} motivates the following approach: estimate the function $g(x,\beta)$ defined in~\eqref{eq:def_g}, search for $\beta^*$ through the criterion $\EE[g(X,\beta^*)]=1-\alpha$, and then plug these estimators into~\eqref{eq:def_opt_sets}. 
Throughout this section, we work under the uniqueness assumption in Theorem~\ref{thm:unified}, namely that the maximizer of $s\mapsto \theta(x,s)+\beta s$ is unique almost surely.  
Our algorithm estimates these quantities to approximate the optimal solution while retaining a rigorous finite-sample PCC guarantee.


Following the standard setups in counterfactual inference~\citep{imbens2015causal,lei2021conformal,jin2023sensitivity}, we assume access to a dataset of i.i.d.~triplets $\{(X_i,Y_i,A_i)\}_{i=1}^N$ where $X_i\sim \PP_X$, and $A_i\sim \pi(\cdot\given X_i)$ where $\pi\colon \cX\to  \Delta(\cA)$ is a behavior policy, and the outcome is from $Y_i\sim \PP_{Y(a)\given X}$ for $a=A_i$. For simplicity, we assume the behavior policy $\pi$ is known; if it is not, it is natural to estimate it and plug into our procedure (which can be viewed as an extension of weighted conformal prediction~\citep{tibshirani2019conformal}). There has been extensive literature on the robustness of  weighted conformal prediction to estimation error in weights~\citep{lei2021conformal} whose results shall similarly apply here; we thus do not pursue this direction.  

We are interested in a test point $X_{\test}\sim \PP_X$ independent of the labeled data, given which the potential outcomes  $\{Y_{\test}(a)\}_{a\in \cA}$ are from $\PP_{Y(a)\given X=X_{\test}}$ for each $a\in \cA$. 
We aim to output prediction sets $\hat{C}(X_{\test},a)\subseteq\cY$ 
that attain the PCC~\eqref{eq:def_P_set} for the induced policy $\pi_{\RA}(X;\hat{C})$. The policy-coupled coverage introduces the new challenge that both $\hat{C}$ and $\pi_{\RA}(\cdot;\hat{C})$ are learned from data and intertwine with each other. 
To achieve the PCC, we take a two-step approach: we first learn certain preliminary prediction sets to pin down the max--min policy, and then calibrate the final prediction sets that induce the same policy while obeying valid coverage. 
 
\vspace{-0.5em}
\paragraph{Data splitting.}
We partition the labeled data into three disjoint sets
$\cI_{\train} \cup \cI_{\learn} \cup \cI_{\calib}=\{1,\dots,N\}$.
The training split is used to fit prediction models; the learning split is used to learn the max--min policy;
the calibration split is used to   construct the final prediction sets.

\vspace{-0.5em}
\paragraph{Step 1: model estimation.}
Recall the definition of the $\alpha$-upper quantile in~\eqref{eq:def_upp_qt}. 
Using the training fold $\{(X_i,Y_i,A_i)\}_{i\in\cI_{\train}}$, we fit 
an estimator $\hat\gamma(x,t,a)$ for utility quantiles $\gamma(x,t,a)=\uquant_{1-t}\big[u(a,Y)\given X=x,A=a\big]$ for  $t\in[0,1]$, $x\in\cX$, and  $a\in\cA$. Here $\hat\gamma(x,t,a)$ is obtained from an outcome model $\widehat{P}(\cdot\given X=x,A=a)$, i.e., analogously to~\cite{kiyani2025decision}, 
\$
\hat\gamma(x,t,a)=\uquant_{1-t}\big[u(a,\widehat Y(a))\given \widehat Y(a)\sim \widehat{P}(\cdot\given X=x,A=a)\big].
\$

We then define  
\$
\hat\theta(x,t):=\max_{a\in\cA}~\hat\gamma(x,t,a),
\qquad
\hat a(x,t):=\argmax_{a\in\cA}~\hat\gamma(x,t,a).
\$

To approximate the optimal solution in Theorem~\ref{thm:optimal-set}, we begin by estimating the optimal conditional coverage assignment rule $t^*(x)=g(x,\beta^*)$. 
Given a Lagrange parameter $\beta\ge 0$, we define
\$
\hat g(x,\beta):=\argmax_{t\in[0,1]} \, \{\hat\theta(x,t)+\beta t\}.
\$
The value of $\beta$ will be calibrated twice. In step 2, we will use $\cI_{\text{learn}}$ to find some $\hat\beta$ that approximately satisfy the coverage to induce an reasonable max--min policy. In step 3, we use conformal inference to calibrate  $\beta$ again with $\cI_{\calib}$, so the prediction sets satisfy the finite-sample coverage guarantee.

\vspace{-0.5em}
\paragraph{Step 2: learning the optimal policy.}
Using the learning fold $\{(X_i,Y_i,A_i)\}_{i\in\cI_{\learn}}$, we find the smallest $\beta$ that satisfies an approximate average coverage constraint:
\@\label{eq:def_hat_beta}
\hat\beta:=\inf\Big\{\beta\ge 0:\ \textstyle{\frac{1}{|\cI_{\learn}|}\sum_{i\in\cI_{\learn}}}\hat g(X_i,\beta) \ge  1-\alpha\Big\}.
\@
Here, we set $\hat\beta$ which estimates $\beta^*$ by enforcing the estimated marginal coverage: $\EE[t^*(X)] = \EE[g(X,\beta)] \approx \hat\EE_{\learn}[\hat{g}(X,\beta)]\geq 1-\alpha$. This typically yields reasonable policies, though it does not necessarily come with any finite-sample guarantee. 
The resulting learned policy is then 
\$
\hat a(x):=\hat a\big(x,\hat t(x)\big),\quad \text{where}\quad 
\hat t(x):=\hat g(x,\hat\beta).
\$

The key idea to disentangle the prediction sets and the induced policy is to fix the latter: in the next step, we calibrate the prediction sets within a class of sets whose max--min rules are $\hat{a}(\cdot)$.

\vspace{-0.5em}
\paragraph{Step 3: weighted conformal calibration.}
For calibration purposes,   we only use labeled data whose logged action matches the learned policy action:
$ 
\cI_{\calib}^0:= \big\{i\in\cI_{\calib}\colon  A_i=\hat a(X_i)\big\}.
$
For any $\beta\ge 0$ and $i\in\cI_{\calib}^0$, we define the calibration score
$
S_i(\beta):=u(A_i,Y_i)\;-\;\hat\theta(X_i,\hat g(X_i,\beta)).
$
Due to sampling under the behavior policy, there is a covariate shift from $\cI_\calib^0$ and the test point (see, e.g.,~\cite{lei2021conformal,tibshirani2019conformal}). 
This can be adjusted by the importance weights based on the known behavior policy (which can be estimated and plugged in if unknown):
\$
w_i:=   \pi(\hat{a}(X_i)\given X_i)^{-1} =  \pi(A_i\given X_i)^{-1},
\qquad
w_{\test} = \pi(\hat a(X_{\test})\given X_{\test})^{-1}.
\$
We then define the calibrated parameter 
\@\label{eq:beta_star}
\beta^*:=
\inf\Bigg\{ \beta\geq 0\colon 
\frac{\sum_{i\in\cI_{\calib}^0} w_i\ind\{S_i(\beta)\ge 0\}}
{\sum_{i\in\cI_{\calib}^0} w_i + w_{\test}}
\ \ge\ 1-\alpha
\Bigg\}.
\@  
Finally, the prediction sets are given by 
\@\label{eq:est_pred_set}
&\hat{C}(X_{\test},\hat{a}(X_{\test}))= \Big\{y\in\cY:\ u(\hat a(X_{\test}),y) \geq \hat\theta(X_{\test},\hat g(X_{\test},\beta^*)) \Big\}\\
&\hat{C}(X_{\test},a)=\Big\{y\in\cY  \colon u(a,y) \geq  \hat\gamma(X_{\test},\hat g(X_{\test},\beta^*),a)  \Big\}, \quad \forall ~ a \neq \hat a(X_\test). \notag 
\@
Note that by definition $\hat{a}(X_{\test}) = \hat\pi_{\RA}(X_{\test})$; see Appendix~\ref{app:subsec_cont_conv} for a formal proof.

The entire procedure is summarized in Algorithm~\ref{alg:two_steps_short_betastar}.

Our method applies to finite and infinite label space. Compared with the usual weighted calibration~\citep{tibshirani2019conformal}, we drop $w_{\text{test}}$ in the numerator; such slight  conservativeness avoids the enumeration of the label space $y\in \cY$ such as in~\cite{kiyani2025decision} and full conformal prediction~\citep{vovk2005algorithmic}. 
When $|\cY|<\infty$, enumeration is feasible and we discuss this below. 
\begin{remark}[Exact calibration for finite label space]\label{rem:orc_cal}
    When $|\cY|<\infty$, one may use \@\label{eq:def_cp_set}
&\hat{C}^{\text{full}}(X_{\test}, \hat{a}(X_{\test})) = \Big\{y\in \cY\colon u(\hat a(X_{\test}),y) \geq \hat\theta(X_{\test},\hat g(X_{\test},\beta^{\text{full}}(y)) \Big\}, \\
&\hat{C}^{\text{full}}(X_{\test},a) = \Big\{y\in \cY\colon u(a,y)\geq \hat\gamma(X_{\test},\hat g(X_{\test},\beta^{\text{full}}_0),a) \Big\},\quad \forall ~ a\neq \hat{a}(X_{\test}), \notag
\@
where $\beta_0^{\text{full}}=\max_{y\in \cY} \beta^{\text{full}}(y)$, and for each $y\in \cY$, $S_{\test}^{(y)}(\beta)=u(\hat a(X_{\test}),y)-\hat\theta(X_{\test},\hat g(X_{\test},\beta))$, 
\@\label{eq:beta_star_y}
\beta^{\text{full}}(y):=
 \inf \Bigg\{ \beta\geq 0\colon 
\frac{\sum_{i\in\cI_{\calib}^0} w_i\ind \{S_i(\beta)\ge 0\} + w_{\test}\ind\{ S_{\test}^{(y)}(\beta)\geq 0\}}
{\sum_{i\in\cI_{\calib}^0} w_i + w_{\test}}
\ \ge\ 1-\alpha
\Bigg\}.
\@ 
The theory of weighted conformal prediction~\citep{tibshirani2019conformal} implies valid coverage of~\eqref{eq:def_cp_set} for $Y_{\test}(\hat{a}(X_{\test}))$, where $\hat a(X_{\test})$ coincides with the risk-averse policy for $\hat{C}^{\text{full}}$; see Theorem~\ref{thm:continuous-cov}. 
\end{remark}

Theorem~\ref{thm:continuous-cov} shows that both prediction sets constructed above satisfy policy-coupled coverage guarantee. 
The proof is in Appendix~\ref{app:subsec_cont_conv}. 

\begin{theorem}\label{thm:continuous-cov}
    Denote the realized outcomes $Y_{\test} = Y_{\test}(\pi_{\RA}(X_{\test};\hat{C}))$ for $\{\hat{C}(x,a)\}_{a\in \cA}$ defined in~\eqref{eq:est_pred_set}, and $Y_{\test}^{\text{full}} = Y_{\test}(\pi_{\RA}(X_{\test};\hat{C}^{\text{full}}))$ for $\{\hat{C}^{\text{full}}(x,a)\}_{a\in \cA}$ defined in~\eqref{eq:def_cp_set}. 
    Then, we have 
    \$
    \PP\big(Y_{\test}^{\text{full}}\in \hat{C}^{\text{full}}(X_{\test},\pi_{\RA}(X_{\test};\hat{C}^{\text{full}}))\big)\geq 1-\alpha,\quad \text{and}\quad \PP\big(Y_{\test}\in \hat{C}(X_{\test},\pi_{\RA}(X_{\test};\hat{C}))\big)\geq 1-\alpha.
    \$
\end{theorem}

\begin{algorithm}[H]
  \caption{Policy-Coupled Risk-Averse Conformal Prediction}
  \label{alg:two_steps_short_betastar}
  \begin{algorithmic}[1]
    \REQUIRE Miscoverage level $\alpha\in(0,1)$; utility $u(\cdot,\cdot)$; logged data $\{(X_i,Y_i,A_i)\}_{i=1}^N$; test covariate $X_{\test}$.

    \STATE Split indices into disjoint sets $\cI_{\train},\cI_{\learn},\cI_{\calib}$.

    \STATE Fit $\hat\gamma(x,t,a)$ on $\cI_{\train}$. \hfill \textcolor{gray}{\texttt{// Nuisance estimation}}

    \STATE Compute $\hat\beta$ on $\cI_{\learn}$ following~\eqref{eq:def_hat_beta}. \hfill \textcolor{gray}{\texttt{// Learn optimal policy}}\\

    \vspace{0.1em}
    \textcolor{gray}{\texttt{// Weighted conformal calibration}}\vspace{0.1em}
    
    \STATE Set $\cI_{\calib}^0=\{i\in\cI_{\calib}:A_i=\hat a(X_i)\}$ and define $S_i(\beta):=u(A_i,Y_i)\;-\;\hat\theta(X_i,\hat g(X_i,\beta)).$ 
    \STATE Set weights $w_i=\frac{1}{ \pi(A_i\given X_i)}$ for $i\in \cI_\calib^0$ and $w_{\test}=\frac{1}{ \pi(\hat a(X_{\test})\given X_{\test})}.$

    \STATE Compute $\beta^*$ as in~\eqref{eq:beta_star} (or $\beta^{\text{full}}(y)$ in~\eqref{eq:beta_star_y} if $|\cY|<\infty$).

    \STATE Compute $\hat{C}(X_{\test},a)$ for $a\in \cA$ based on~\eqref{eq:est_pred_set} (or $\hat{C}^{\text{full}}(X_{\test},a)$ based on~\eqref{eq:def_cp_set}).  
    
    \ENSURE Risk-averse policy $\hat a(\cdot)$ and prediction sets $\{\hat{C}(X_{\test},a)\}_{a\in\cA}$ (or $\{\hat{C}^{\text{full}}(X_{\test},a)\}_{a\in\cA}$).
  \end{algorithmic}
\end{algorithm}


\section{Simulation study}
\label{sec:simulation}

We first evaluate the proposed \textsc{PC-RACP} method on a synthetic decision-making task. The goal is to assess whether modeling the action-conditional outcome distribution improves the decision quality while maintaining the target coverage guarantee.

\vspace{-0.5em}
\paragraph{Task and data generation process.} The covariates $X\in\RR^d$ are i.i.d.~from a multivariate standard Gaussian distribution. Conditional on $X$, the logged action $A$ is drawn from a softmax behavior policy $\pi_b(a\given X)$. The outcome $Y$ is then sampled from an action-dependent conditional distribution $\PP_{Y(a)\given X}$ for $a=A$, also parameterized by a softmax model. See Appendix~\ref{app:simulation-dgp} for details. We set $|\cA|=3$ and $|\cY|=4$. The utility function is  
\$
u(a,y)=
\begin{array}{c|cccc}
 & y=0 & y=1 & y=2 & y=3\\
\hline
a=0 & 0.70 & 0.60 & 0.35 & 0.15\\
a=1 & 0.95 & 0.55 & 0.45 & 0.15\\
a=2 & 0.80 & 0.50 & 0.20 & 0.20
\end{array}.
\$
We also set $u_{\max}=1.00$. 
This utility specification induces nontrivial trade-offs across actions and labels, making the quality of the prediction set directly relevant to decisions. We generate $30{,}000$ i.i.d.~samples in total, partition them into disjoint subsets  
(30\%\ \text{train},  
20\%\ \text{learn},  
20\%\ \text{calib},  
30\%\ \text{test}), and apply the methods to produce the learned actions and prediction sets. The experiments are repeated for $N=20$ independent runs. 

\vspace{-0.5em}
\paragraph{Methods.}
We compare the proposed \textsc{PC-RACP} procedure with RAC (\cite{kiyani2025decision}) and a Plug-in baseline. RAC constructs prediction sets from the marginal predictive distribution $\PP(Y\given X)$, without conditioning on the action. The Plug-in baseline learns the action-conditional model $\PP(Y\given X,A)$, but does not perform any calibration: it directly selects the action by maximizing the plug-in utility quantile and then forms the prediction set; it is detailed in Appendix~\ref{app:plugin}.

Although the true behavior policy is known in the synthetic setup, we use estimated behavior policies whenever behavior-policy information is required, in order to mimic the logged-data setting and examine robustness to behavior-policy estimation error.

In PC-RACP, we fit the models on the training split via multinomial regression, including the behavior policy model $\widehat{\PP}(A\given X)$ and  outcome model $\widehat{\PP}(Y\given X,A)$ which regresses $Y$ on $X$ using  the training samples with observed action $A=a$. In RAC, we ignore the actions and fit the marginal predictive model $\widehat{\PP}(Y\given X)$ via a multinomial regression of $Y$ on $X$ using the training split; the pooled learn and calibration splits are then used for RAC calibration. The Plug-in method fits multinomial regression for each action $a\in \cA$ in the pooled train, learn, and calib splits. In this way, the three methods use the same set of labeled data. 
Finally, in addition to the well-specified multinomial regression, we conduct the same set of experiments using the random forest classifier, while keeping the same data splits and model fitting protocol as above.

\vspace{-0.5em}
\paragraph{Evaluation metrics.}  

We vary the target miscoverage level over
$
\alpha \in \{0.02,0.04,\dots,0.20\}.
$
For each test point $X_i$, PC-RACP and the Plug-in method output a learned action $\hat a_i$ and a prediction set $\hat C_i=\hat{C}(X_i,\hat{a}_i)$. For RAC, the method outputs a prediction set $\hat C_i$, and the action is then induced by $\hat a_i=\argmax_{a\in\cA}\min_{y\in \hat C_i}u(a,y)$. We evaluate (i) the empirical coverage by 
$ \widehat{\mathrm{Cov}}(\alpha)=\frac{1}{|I_{\test}|}\sum_{i\in I_{\test}}\ind\! \{  Y_i(\hat a_i)\in \hat C_i \} 
$ based on the sampled counterfactual outcomes, and (ii) the utility certificate by $\widehat{\mathrm{Util}}(\alpha)=\frac{1}{|I_{\test}|}\sum_{i\in I_{\test}}\min_{y\in \widehat C_i} u(\hat a_i,y)$ which serves as a valid lower bound for the realized utility when coverage is satisfied.

\begin{figure}
    \centering

    \begin{subfigure}[t]{0.49\linewidth}
        \centering
        \includegraphics[width=\linewidth]{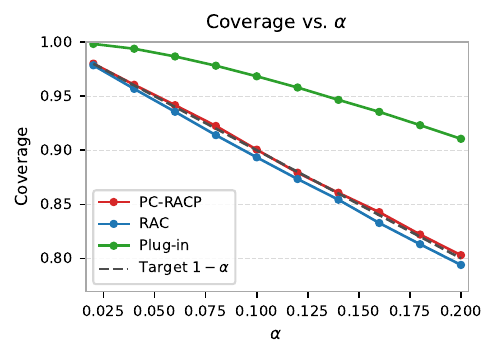}
        \caption{Multinomial model: empirical coverage v.s.~$\alpha$.}
        \label{fig:coverage-simu-1-logistic}
    \end{subfigure}
    \hfill
    \begin{subfigure}[t]{0.49\linewidth}
        \centering
        \includegraphics[width=\linewidth]{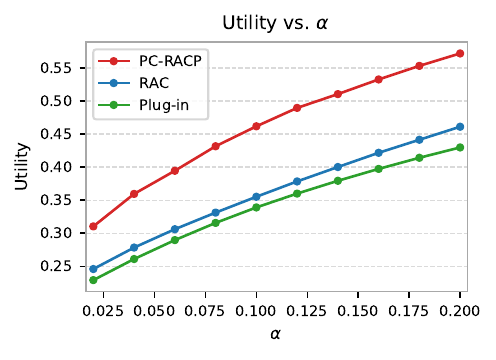}
        \caption{Multinomial model: average utility certificate v.s.~$\alpha$.}
        \label{fig:utility-simu-1-logistic}
    \end{subfigure}

    \vspace{0.6em}

    \begin{subfigure}[t]{0.49\linewidth}
        \centering
        \includegraphics[width=\linewidth]{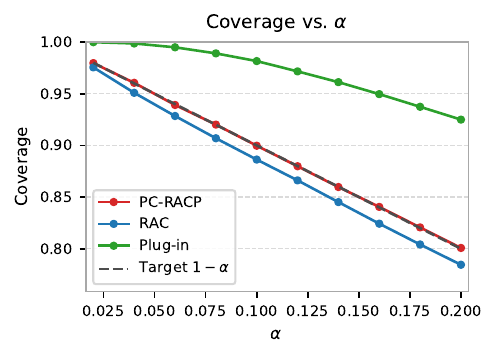}
        \caption{Random forest: empirical coverage v.s.~$\alpha$.}
        \label{fig:coverage-simu-1-rf}
    \end{subfigure}
    \hfill
    \begin{subfigure}[t]{0.49\linewidth}
        \centering
        \includegraphics[width=\linewidth]{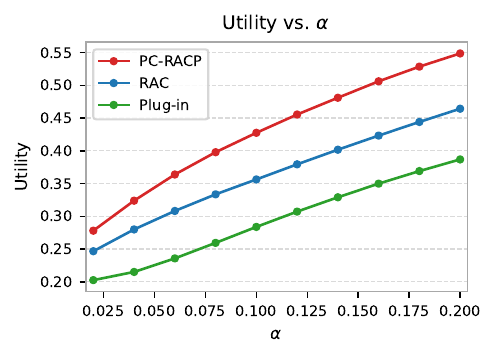}
        \caption{Random forest: average utility certificate v.s.~$\alpha$.}
        \label{fig:utility-simu-1-rf}
    \end{subfigure}

    \caption{Simulation results across various miscoverage levels $\alpha$. Top row show results when models are fitted by multinomial regression. Bottom row show results when models are fitted by random forests. Left column: coverage across different values of $\alpha$, where the dashed line indicating the target level $1-\alpha$. Right column: average utility certificate across different values of $\alpha$.}
    \label{fig:sim-results}
\end{figure}

\vspace{-0.5em}
\paragraph{Results.} The simulation results are summarized in Figure~\ref{fig:sim-results}. The empirical coverage in the left panel shows the tight coverage of \textsc{PC-RACP} for the outcomes under set-induced actions. 
However, RAC falls slightly below the target curve: it ignores the counterfactual nature, so the actions would induce a distinct distribution for the realized outcome compared with the marginal distribution of $Y(A)$ mixed under the behavior policy. On the other hand, the Plug-in baseline fails to provide exact coverage, even though we expect the learned models are more accurate since more samples are used in model fitting; this shows the importance of decision-oriented conformal calibration for finite-sample coverage.

The right panel of Figure~\ref{fig:sim-results} presents the average utility certificate for each method. \textsc{PC-RACP} consistently achieves the highest average utility certificate across $\alpha$. Compared with the RAC method, this shows the benefit of policy-coupled inference in counterfactual settings. Compared with the Plug-in baseline which performs the worst, this shows the importance of calibration: even with a well-specified multinomial model, the learning step (which ensures approximate coverage before deriving the policy) and conformal calibration (which ensures exact finite-sample coverage) are still instrumental to realize higher utility. Overall, these results confirm the strong empirical performance of \textsc{PC-RACP} in both safety (coverage) and utility. 

\section{Hillstrom experiment}
\label{sec:hillstrom-exp}

We next evaluate the proposed \textsc{PC-RACP} method on the Hillstrom email marketing dataset~\citep{minethatdata} to investigate whether the utility gains observed in simulation persist in a real-data decision problem. 
This is a randomized experiment for e-mail merchandise campaign. In online marketing campaigns, different interventions (such as sending a promotion email versus not) lead to distinct customer behavior and downstream purchasing outcomes, so the counterfactual perspective is necessary and commonly adopted as the standard~\citep{varian2016causal}. 

\subsection{Experimental setup}

\paragraph{Task, data preprocessing, and model training.}
The original treatment variable \texttt{segment} has three actions; for easier interpretation, we map it to a binary action space $A \in \{0,1\}$,
where $A=0$ denotes \texttt{No E-Mail} and $A=1$ denotes sending an email advertisement (merging \texttt{Mens E-Mail} and \texttt{Womens E-Mail}). The outcome is the binary label $Y \in \{0,1\}$,
where $Y=1$ indicates a website visit and $Y=0$ otherwise. 
We use five user features $X = (\texttt{recency},\texttt{history},\texttt{mens},\texttt{womens},\texttt{newbie})$ as the covariates. Here, \texttt{recency} and \texttt{history} are numeric-valued customer-history variables,  
while \texttt{mens}, \texttt{womens}, and \texttt{newbie} are binary variables for purchasing history of men and women products, and new customer. 
The behavior policy in Hillstrom is known: 
$\PP(A=0 \given X)=\tfrac13$ and $\PP(A=1 \given X)=\tfrac23$ due to the randomization. 
We set the utility function
\$
u(a,y)=
\begin{array}{c|cc}
 & y=0 & y=1\\
\hline
a=0 & 0.40 & 0.25\\
a=1 & 0.10 & 0.90
\end{array} 
\$
to encode the objective of the email marketing task: emails should ideally be sent to users who are likely to visit the website after receiving them. Accordingly, the case $(a=1,y=1)$ is assigned the highest utility, since it corresponds to a successful email campaign. The case $(a=1,y=0)$ is assigned low utility because it represents an unsuccessful email and thus a wastage of resources. The case $(a=0,y=0)$ receives relatively high utility, as refraining from sending an email is appropriate for a user who would not visit. The case $(a=0,y=1)$ is also undesirable because it reflects insufficient intervention: a potentially responsive user is not contacted. This missed-opportunity case is treated as less harmful than sending an unnecessary email. We also set $u_{\max}=1.00$.

\vspace{-0.5em}
\paragraph{Methods and evaluation.}
The Hillstrom dataset contains $n=64{,}000$ samples. We split the data into
30\%\ \text{train},  
20\%\ \text{learn},  
20\%\ \text{calib}, and 
30\%\ \text{test},
following the same protocol as in the simulation study.   
On the training split, we fit the action-conditional outcome model $\hat P(Y \given X,A)$ and the marginal label model $\hat P(Y \given X)$ using CatBoost. For $\hat P(Y \given X,A)$, we train a separate model for each action $a\in\{0,1\}$. We use  the known logging policy given above. 

We vary the miscoverage level 
$
\alpha \in \{0.02,0.04,\dots,0.20\}.
$
We report empirical coverage and average utility on the test set for each method. The average utility is computed in the same way as in the simulation study. Although the counterfactual outcomes are unknown, we can still construct an unbiased estimator for coverage; see Appendix~\ref{app:subsec_eval_real}. 

\subsection{Main results}

\begin{figure}
    \centering
    \begin{subfigure}[t]{0.49\linewidth}
        \centering
        \includegraphics[width=\linewidth]{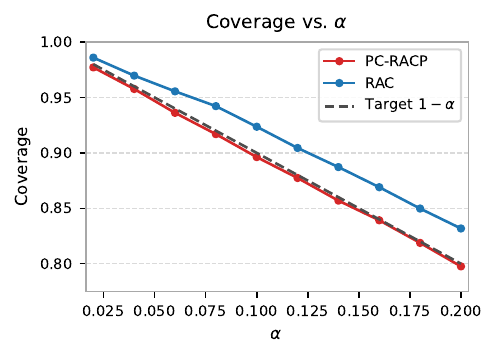}
        \caption{Empirical coverage versus $\alpha$.}
        \label{fig:hillstrom-coverage}
    \end{subfigure}
    \hfill
    \begin{subfigure}[t]{0.49\linewidth}
        \centering
        \includegraphics[width=\linewidth]{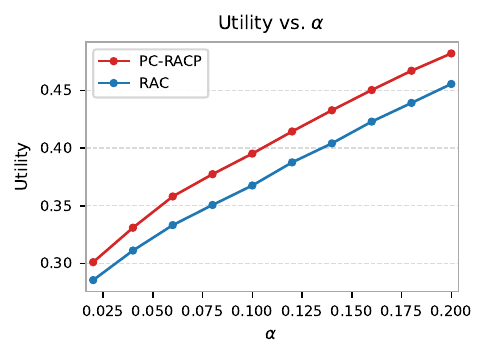}
        \caption{Average utility certificate versus $\alpha$.}
        \label{fig:hillstrom-utility}
    \end{subfigure}
    \caption{Results on the Hillstrom dataset. Left: empirical coverage across $\alpha$. Right: estimated average utility certificate across $\alpha$. \textsc{PC-RACP} consistently achieves  higher utility than RAC while both methods keep empirical coverage close to the nominal target $1-\alpha$.}
    \label{fig:hillstrom-main}
\end{figure}

The (estimated) empirical coverage and utility certificate are reported in Figure~\ref{fig:hillstrom-main}. 
Figure~\ref{fig:hillstrom-coverage} shows that \textsc{PC-RACP} maintains empirical coverage close to the nominal level $1-\alpha$ across the full range of $\alpha$. In contrast, RAC exhibits empirical coverage that is consistently above the target, indicating that it is more conservative than required. 
Figure~\ref{fig:hillstrom-utility} shows that \textsc{PC-RACP} consistently achieves higher utility than RAC. Together, these results show that the main simulation takeaway continues to hold on Hillstrom: explicitly modeling action-dependent outcome uncertainty can improve decision quality, while maintaining valid coverage for the realized outcome.

\subsection{Where the utility gain comes from}

We now zoom into the decisions produced by \textsc{PC-RACP} to analyze the utility gain over RAC.  
Let $\hat{A}\in \{0,1\}$ denote the decision produced by the methods. 
Although the general construction allows $\hat{C}=\emptyset$, empty prediction sets do not occur in this experiment; hence we omit this case from the table below.
For the nonempty prediction sets $\{0,1\}$, $\{0\}$, and $\{1\}$, the above utility function gives the following values of $\min_{y\in \hat{C}}u(a,y)$ for each configuration of action $a$ and prediction set $\hat{C}$:
\@\label{eq:util_bd_table}
\min_{y\in \hat{C}}u(a,y) =
\begin{array}{c|ccc}
& \hat{C}=\{0\} & \hat{C}=\{1\} & \hat{C}=\{0,1\}\\
\hline
a=0 & 0.40 & 0.25 & 0.25\\
a=1 & 0.10 & 0.90 & 0.10
\end{array} .
\@
Ignoring the counterfactual nature, RAC produces one single prediction set $\hat{C}$, based on which the action is selected. The RAC-selected action is $\hat{A}=0$, $\hat{A}=1$, and $\hat{A}=0$ for the three columns in~\eqref{eq:util_bd_table}. In contrast, \textsc{PC-RACP} constructs separate $\hat{C}(0)$ and $\hat{C}(1)$ for the two actions, so the selected action $\hat{A}$ depends on the comparison of two cells in the two rows of~\eqref{eq:util_bd_table}.

\begin{figure}[t]
    \centering
    \begin{subfigure}[t]{0.49\linewidth}
        \centering
        \includegraphics[width=\linewidth]{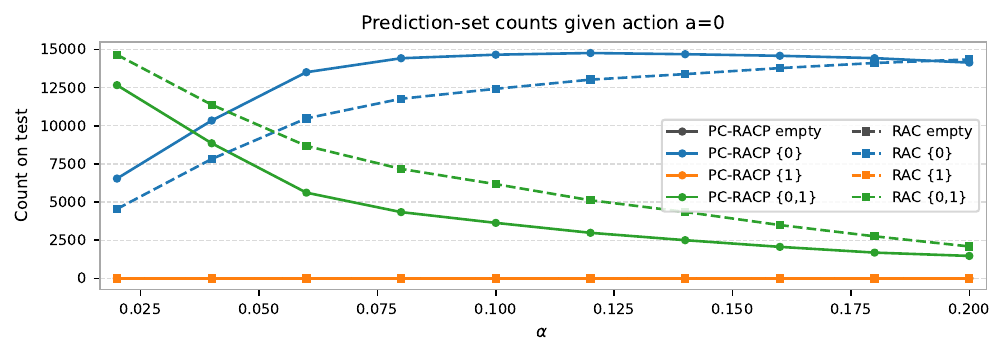}
        \caption{Prediction-set counts among test points with selected action $\hat{A}=0$.}
        \label{fig:hillstrom-pair-a0}
    \end{subfigure}
    \hfill
    \begin{subfigure}[t]{0.49\linewidth}
        \centering
        \includegraphics[width=\linewidth]{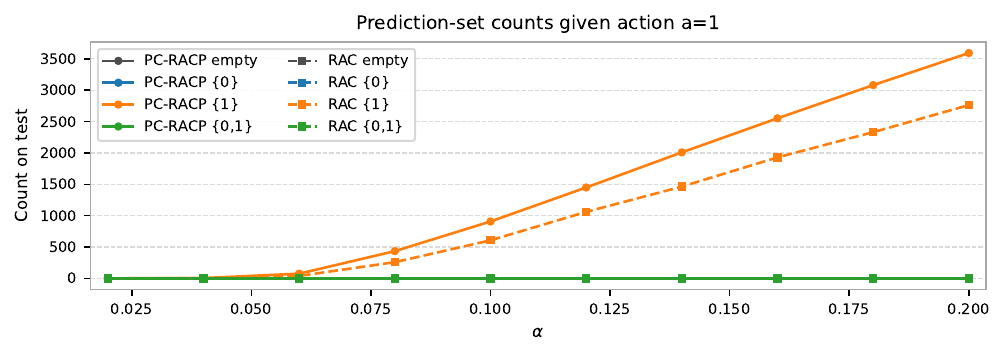}
        \caption{Prediction-set counts among test points with selected action $\hat{A}=1$.}
        \label{fig:hillstrom-pair-a1}
    \end{subfigure}
    \caption{Prediction-set composition conditional on the selected action. For $\hat{A}=0$, the prediction sets concentrate on $\{0\}$ and $\{0,1\}$. For $\hat{A}=1$, the prediction sets are essentially all equal to $\{1\}$.}
    \label{fig:hillstrom-pair-counts}
\end{figure}

\begin{figure}
    \centering
    \begin{subfigure}[t]{0.49\linewidth}
        \centering
        \includegraphics[width=\linewidth]{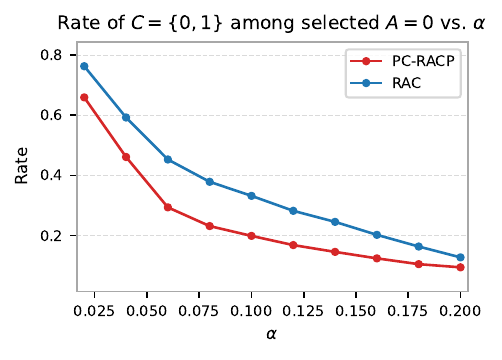}
        \caption{Fraction of $\hat{C}(X,\hat{A})=\{0,1\}$ among those $\hat{A}=0$.}
        \label{fig:hillstrom-a0-rate}
    \end{subfigure}
    \hfill
    \begin{subfigure}[t]{0.49\linewidth}
        \centering
        \includegraphics[width=\linewidth]{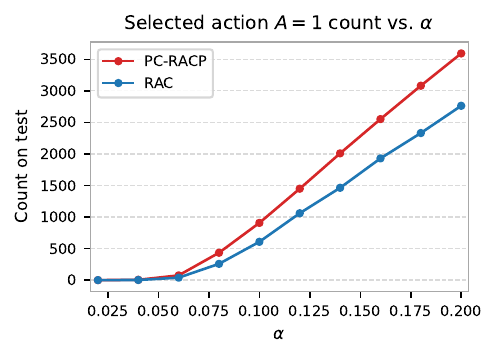}
        \caption{Number of test points with final action $\hat{A}=1$.}
        \label{fig:hillstrom-a1-count}
    \end{subfigure}
    \caption{Left: in the small-$\alpha$ regime, among test points with $\hat{A}=0$, \textsc{PC-RACP} produces prediction set $\{0,1\}$ for a substantially smaller fraction of cases than RAC. Right: in the large-$\alpha$ regime, \textsc{PC-RACP} selects more $\hat{A}=1$ cases than RAC. 
    }
    \label{fig:hillstrom-mechanism}
\end{figure}

First, we examine the form of the prediction sets conditional on the selected action (for \textsc{PC-RACP} we examine the prediction set for the selected action). Figure~\ref{fig:hillstrom-pair-a0} shows that, for both methods, test samples with $\hat{A}=0$ are assigned with prediction sets of the form $\{0\}$ or $\{0,1\}$. In contrast, Figure~\ref{fig:hillstrom-pair-a1} shows that when the selected action is $\hat{A}=1$, the prediction set always takes the form $\{1\}$. Thus, the sharpness of the utility certification relies on two factors: (i) how frequently the test samples are assigned to $\hat{A}=1$, and (ii) when $\hat{A}=0$, how frequently the prediction set takes the form $\{0\}$ instead of $\{0,1\}$. Figure~\ref{fig:hillstrom-mechanism} explains these two cases:
\begin{itemize}
\item 
For a small value of $\alpha$, the comparison is driven mainly by the cases with $\hat{A}=0$ (very few test points are assigned $\hat{A}=1$ by both methods;   see Figure~\ref{fig:hillstrom-a1-count}).  Figure~\ref{fig:hillstrom-a0-rate} shows that for these cases, the fraction of ambiguous sets $\{0,1\}$ produced by \textsc{PC-RACP} is substantially smaller than that produced by RAC, which leads to higher utility certificate: for action $a=0$, the utility certificate is $\min\{u(0,0),u(0,1)\}=0.25$ when $\hat{C}=\{0,1\}$, whereas  $\hat{C}=\{0\}$ gives a sharper utility certificate of $u(0,0)=0.40$. 
\item
When $\alpha$ becomes larger, the mechanism changes. Figure~\ref{fig:hillstrom-a1-count} shows that \textsc{PC-RACP} assigns substantially more test points to the  action $\hat{A}=1$ than RAC, and the gap widens as $\alpha$ increases. This improves the utility because, as shown in Figure~\ref{fig:hillstrom-pair-a1}, the prediction set for those $\hat{A}=1$ is essentially always $\{1\}$ with a high utility certificate $u(1,1)=0.90$. Hence, in the large-$\alpha$ regime, action selection plays an important role in sharpness. 
\end{itemize}

Taken together, these figures show that the utility gain of \textsc{PC-RACP} comes from reducing conservativeness in two different ways. For small $\alpha$, it is less conservative in prediction-set construction, more frequently producing sets of the form $\{0\}$ for the dominant $\hat{A}=0$ cases.  For larger values of $\alpha$, it is less conservative in action selection, assigning $\hat{A}=1$ to more users than RAC.

\section*{Discussion}

We developed a decision-theoretic framework for predictive inference in counterfactual decision-making. The key challenge is that, when actions determine outcomes, the target of coverage is no longer fixed in advance. This makes ordinary validity notions insufficient for understanding how prediction sets should support downstream decisions.
Our results identify policy-coupled coverage as the decision-relevant validity notion, where the coverage is evaluated on the potential outcome realized under the policy induced by the prediction sets themselves. For a given collection of per-action prediction sets, it justifies the counterfactual max--min rule for risk-averse decisions under distributional ambiguity. Furthermore, optimizing prediction sets under policy-coupled coverage is equivalent both to imposing universal policy coverage and to directly optimizing over policies and utility certificates. Thus, policy-coupled prediction sets form a lossless interface between uncertainty and action.
Overall, our results suggest that uncertainty quantification in counterfactual problems should be designed around the decisions it supports. We view policy-coupled coverage as a first step toward a broader theory of decision-aware uncertainty quantification for counterfactual settings. 

\bibliographystyle{apalike}
\bibliography{reference}

@book{vovk2005algorithmic,
  title={Algorithmic learning in a random world},
  author={Vovk, Vladimir and Gammerman, Alexander and Shafer, Glenn},
  volume={29},
  year={2005},
  publisher={Springer}
}

@article{kiyani2025decision,
  title={Decision theoretic foundations for conformal prediction: Optimal uncertainty quantification for risk-averse agents},
  author={Kiyani, Shayan and Pappas, George and Roth, Aaron and Hassani, Hamed},
  journal={arXiv preprint arXiv:2502.02561},
  year={2025}
}

@article{foygel2021limits,
  title={The limits of distribution-free conditional predictive inference},
  author={Foygel Barber, Rina and Candes, Emmanuel J and Ramdas, Aaditya and Tibshirani, Ryan J},
  journal={Information and Inference: A Journal of the IMA},
  volume={10},
  number={2},
  pages={455--482},
  year={2021},
  publisher={Oxford University Press}
}

@article{tibshirani2019conformal,
  title={Conformal prediction under covariate shift},
  author={Tibshirani, Ryan J and Foygel Barber, Rina and Candes, Emmanuel and Ramdas, Aaditya},
  journal={Advances in neural information processing systems},
  volume={32},
  year={2019}
}

@article{zhao2021calibrating,
  title={Calibrating predictions to decisions: A novel approach to multi-class calibration},
  author={Zhao, Shengjia and Kim, Michael and Sahoo, Roshni and Ma, Tengyu and Ermon, Stefano},
  journal={Advances in Neural Information Processing Systems},
  volume={34},
  pages={22313--22324},
  year={2021}
}

@article{rubin2005causal,
  title={Causal inference using potential outcomes: Design, modeling, decisions},
  author={Rubin, Donald B},
  journal={Journal of the American statistical Association},
  volume={100},
  number={469},
  pages={322--331},
  year={2005},
  publisher={Taylor \& Francis}
}

@book{imbens2015causal,
  title={Causal inference in statistics, social, and biomedical sciences},
  author={Imbens, Guido W and Rubin, Donald B},
  year={2015},
  publisher={Cambridge university press}
}

@article{wang2026optimal,
  title={Optimal Decision-Making Based on Prediction Sets},
  author={Wang, Tao and Dobriban, Edgar},
  journal={arXiv preprint arXiv:2602.00989},
  year={2026}
}

@article{lei2021conformal,
  title={Conformal inference of counterfactuals and individual treatment effects},
  author={Lei, Lihua and Cand{\`e}s, Emmanuel J},
  journal={Journal of the Royal Statistical Society Series B: Statistical Methodology},
  volume={83},
  number={5},
  pages={911--938},
  year={2021},
  publisher={Oxford University Press}
}

@article{jin2023sensitivity,
  title={Sensitivity analysis of individual treatment effects: A robust conformal inference approach},
  author={Jin, Ying and Ren, Zhimei and Cand{\`e}s, Emmanuel J},
  journal={Proceedings of the National Academy of Sciences},
  volume={120},
  number={6},
  pages={e2214889120},
  year={2023},
  publisher={National Academy of Sciences}
}

@online{minethatdata,
  author = {Hillstrom, Kevin},
  title = {The MineThatData E-Mail Analytics And Data Mining Challenge},
  url = {https://blog.minethatdata.com/2008/03/minethatdata-e-mail-analytics-and-data.html},
  year = {2008}, 
  urldate = {2026-06-24} 
}

@article{varian2016causal,
  title={Causal inference in economics and marketing},
  author={Varian, Hal R},
  journal={Proceedings of the National Academy of Sciences},
  volume={113},
  number={27},
  pages={7310--7315},
  year={2016},
  publisher={National Academy of Sciences}
}

@article{yin2024conformal,
  title={Conformal sensitivity analysis for individual treatment effects},
  author={Yin, Mingzhang and Shi, Claudia and Wang, Yixin and Blei, David M},
  journal={Journal of the American Statistical Association},
  volume={119},
  number={545},
  pages={122--135},
  year={2024},
  publisher={Taylor \& Francis}
}

@misc{zhu2026conformalriskaversedecisionmaking,
      title={Conformal Risk-Averse Decision Making with Action Conditional Guarantee}, 
      author={Zihan Zhu and Shayan Kiyani and George Pappas. Hamed Hassani},
      year={2026},
      eprint={2606.05551},
      archivePrefix={arXiv},
      primaryClass={stat.ML},
      url={https://arxiv.org/abs/2606.05551}, 
}

@article{alaa2023conformal,
  title={Conformal meta-learners for predictive inference of individual treatment effects},
  author={Alaa, Ahmed M and Ahmad, Zaid and van der Laan, Mark},
  journal={Advances in neural information processing systems},
  volume={36},
  pages={47682--47703},
  year={2023}
}

@article{foster1998asymptotic,
  title={Asymptotic calibration},
  author={Foster, Dean P and Vohra, Rakesh V},
  journal={Biometrika},
  volume={85},
  number={2},
  pages={379--390},
  year={1998},
  publisher={Oxford University Press}
}

@article{kakade2008deterministic,
  title={Deterministic calibration and Nash equilibrium},
  author={Kakade, Sham M and Foster, Dean P},
  journal={Journal of Computer and System Sciences},
  volume={74},
  number={1},
  pages={115--130},
  year={2008},
  publisher={Elsevier}
}

@inproceedings{athey2015machine,
  title={Machine learning and causal inference for policy evaluation},
  author={Athey, Susan},
  booktitle={Proceedings of the 21th ACM SIGKDD international conference on knowledge discovery and data mining},
  pages={5--6},
  year={2015}
}

@article{feuerriegel2024causal,
  title={Causal machine learning for predicting treatment outcomes},
  author={Feuerriegel, Stefan and Frauen, Dennis and Melnychuk, Valentyn and Schweisthal, Jonas and Hess, Konstantin and Curth, Alicia and Bauer, Stefan and Kilbertus, Niki and Kohane, Isaac S and van der Schaar, Mihaela},
  journal={Nature Medicine},
  volume={30},
  number={4},
  pages={958--968},
  year={2024},
  publisher={Nature Publishing Group US New York}
}

@article{gao2024causal,
  title={Causal inference in recommender systems: A survey and future directions},
  author={Gao, Chen and Zheng, Yu and Wang, Wenjie and Feng, Fuli and He, Xiangnan and Li, Yong},
  journal={ACM Transactions on Information Systems},
  volume={42},
  number={4},
  pages={1--32},
  year={2024},
  publisher={ACM New York, NY}
}

@inproceedings{li2010contextual,
  title={A contextual-bandit approach to personalized news article recommendation},
  author={Li, Lihong and Chu, Wei and Langford, John and Schapire, Robert E},
  booktitle={Proceedings of the 19th international conference on World wide web},
  pages={661--670},
  year={2010}
}

@article{manski2004statistical,
  title={Statistical treatment rules for heterogeneous populations},
  author={Manski, Charles F},
  journal={Econometrica},
  volume={72},
  number={4},
  pages={1221--1246},
  year={2004},
  publisher={Wiley Online Library}
}

\newpage 
\appendix 


\section{General results with randomization}
\label{app:sec_random}
In the main text, we discuss the optimization problems for deterministic $\pi$, $\nu$, and $C$ under the assumption that the relevant maximizers are unique. Here, we expand the framework to full generality by allowing randomization of $\pi$, $\nu$, and $C$, without assuming such uniqueness.

Let $\Xi,\Xi_1,\Xi_2\sim\mathrm{Unif}(0,1)$ denote exogenous random seeds, independent of
$(X,\{Y(a)\}_{a\in\cA})$. All maps below are assumed measurable. 
We consider a randomized policy mapping $\pi:\mathcal X\times[0,1]\to\mathcal A$ and a randomized utility mapping $\nu:\cX\times[0,1]\to(-\infty,u_{\max}]$. 
The extra randomness aims to smooth up the threshold conditions in the RA-DPO and RA-CPO problems.

\subsection{Randomized RA-DPO}
First, we introduce the corresponding randomized RA-DPO programs.

\medskip
\noindent\textbf{Randomized RA-DPO (two seeds).}
\begin{equation}\label{eq:RA-DPO-two-seeds}\tag{RA-DPO-Rand-2}
\begin{aligned}
\max_{\nu,\pi,\cL(\Xi_1,\Xi_2)}\quad & \EE\big[\nu(X,\Xi_1)\big] \\
\text{s.t.}\quad
& \PP(u(\pi(X,\Xi_2),\,Y(\pi(X,\Xi_2))) \ge \nu(X,\Xi_1))\ge\ 1-\alpha, \\
& \Xi_1,\Xi_2{\sim}\mathrm{Unif}(0,1),\quad (\Xi_1,\Xi_2)\perp (X,\{Y(a)\}_{a\in\cA}). 
\end{aligned}
\end{equation}
Here, $\cL(\Xi_1,\Xi_2)$ denotes the joint law of the two uniform seeds. All the probabilities are taken over the joint distribution of $(X,\{Y(a)\}_{a\in\cA})$ and the exogenous random seeds. 

While it is natural to separately randomize the policy and the utility certificate, the above problem relying on two random seeds is equivalent to the problem with one seed governing both. This can be seen from the fact that a single Unif$[0,1]$ random variable can simulate any Borel probability distribution on $[0,1]^2$, yet we include a formal statement in Theorem~\ref{thm:equi-seed} with proof in Appendix~\ref{app:subsec_equi_seed}.

\medskip
\noindent\textbf{Randomized RA-DPO (single-seed).}
Consider the following optimization problem with one single seed:
\begin{equation}
\label{eq:RA-DPO-one-seed}\tag{RA-DPO-Rand} 
\begin{aligned}
\max_{\nu,\pi}\quad & \EE[\nu(X,\Xi)] \\
\text{s.t.}\quad
& \PP\big(u(\pi(X,\Xi),\,Y(\pi(X,\Xi)))  \ge \nu(X,\Xi) \big) \ge 1-\alpha,\\
& \Xi\sim\mathrm{Unif}(0,1),\quad \Xi\perp (X,\{Y(a)\}_{a\in\cA}).
\end{aligned}
\end{equation}

Hereafter, we focus on the single-seed RA-DPO-Rand problem as our RA-DPO problem of interest. 
We call a collection of measurable maps together with exogenous seeds attaining the optimum an optimal realization.
\begin{theorem}\label{thm:equi-seed}
 Randomized RA-DPO (two seeds) and randomized RA-DPO (single-seed) are equivalent with each other. From any optimal realization $(\nu_2,\pi_2,\Xi_1,\Xi_2)$ of Randomized RA-DPO (two seeds), we can construct an optimal realization $(\nu_1,\pi_1,\Xi)$ of randomized RA-DPO (single-seed) such that $\EE[\nu_1(X,\Xi)]=\EE[\nu_2(X,\Xi_1)]$. Vice versa,
 From any optimal realization $(\nu_1,\pi_1,\Xi)$ of randomized RA-DPO (single-seed), we can construct an optimal realization $(\nu_2,\pi_2,\Xi_1,\Xi_2)$ of randomized RA-DPO (two seeds) such that $\EE[\nu_1(X,\Xi)]=\EE[\nu_2(X,\Xi_1)]$.  
\end{theorem}

Proposition~\ref{thm:non-randomized} further shows that for any randomized optimal solution $\pi$ to~\ref{eq:RA-DPO-one-seed}, its dependence on $\xi$ is through the value optimizer $\nu(x,\xi)$ paired with it. The proof is in Appendix~\ref{app:subsec_non-randomized}. 

\begin{prop}\label{thm:non-randomized}
Assume RA-DPO-Rand admits an optimal solution $(\nu,\pi)$.
Then there exists an optimal solution $(\nu,\tilde\pi)$ such that
\$
\nu(x,\xi)=\nu(x,\xi') \ \Longrightarrow\ \tilde\pi(x,\xi)=\tilde\pi(x,\xi'),
\qquad \forall x\in\cX,\ \forall \xi,\xi'\in[0,1].
\$
In particular, $\tilde\pi(x,\xi)$ depends on $\xi$ only through the value $\nu(x,\xi)$.
\end{prop}

\subsection{Randomized RA-CPO}

Given prediction sets $\{C(X,a,\xi)\}_{a\in \cA}$, we define the (random-seed-dependent) risk-averse policy $\pi_{\RA}$ and reward $\nu_{\RA}$ through
\@\label{eq:pi-nu-ra-def}
\pi_\RA(x,\xi;C):=\argmax_{a\in\cA} \min_{y\in C(x,a,\xi)} u(a,y),
\qquad
\nu_\RA(x,\xi;C):=\max_{a\in\cA} \min_{y\in C(x,a,\xi)} u(a,y).
\@
We then analogously define the randomized versions of RA-CPO-1 and RA-CPO-2. 

\medskip
\noindent\textbf{Randomized RA-CPO-2 (uniform coverage over all randomized policies).}
Let $C:\cX\times\cA\times[0,1]\to 2^{\cY}$ be a randomized prediction-set rule.
Consider
\begin{equation}\label{eq:RA-CPO2-Xi}\tag{RA-CPO-2-Rand}
\begin{aligned}
  \max_{C}\quad 
& 
\EE_{X,\Xi}[\nu_\RA(X,\Xi;C)]\\
\text{s.t.}\quad
& \PP( Y(\pi(X,\Xi))\in C(X,\pi(X,\Xi),\Xi)) \ge\ 1-\alpha,
~~ \forall\,\pi:\cX\times[0,1]\to\cA, \\
& \Xi\sim\mathrm{Unif}(0,1),\quad \Xi\perp (X,\{Y(a)\}_{a\in\cA}).  
\end{aligned}
\end{equation}

\medskip
\noindent\textbf{Randomized RA-CPO-1 (coverage only along $\pi_\RA$).}
\begin{equation}\label{eq:RA-CPO1-Xi}\tag{RA-CPO-1-Rand}
\begin{aligned}
    \max_{C}\quad 
& 
\EE_{X,\Xi}[\nu_\RA(X,\Xi;C)] \\
\text{s.t.}\quad
& \mathbb P( Y(\pi_\RA(X,\Xi;C))\in C(X,\pi_\RA(X,\Xi;C),\Xi)) \ge\ 1-\alpha, \\
& \Xi\sim\mathrm{Unif}(0,1),\quad \Xi\perp (X,\{Y(a)\}_{a\in\cA}).
\end{aligned}
\end{equation}  
Note that all these programs strictly generalizes the deterministic formulations. 

\subsection{Equivalence between randomized optimization problems}

With randomization, RA-DPO, RA-CPO-1 and RA-CPO-2 remain equivalent (i.e., with the same optimal values). The following two theorems generalize the results in Section~\ref{subsec:opt_equiv}, with proofs in Appendix~\ref{app:subsec_dpo_cpo} and Appendix~\ref{app:subsec_equiv_cpo}, respectively.

\begin{theorem}\label{thm:random-cpo}
  \textnormal{\ref{eq:RA-CPO1-Xi}} and  \textnormal{\ref{eq:RA-CPO2-Xi}} are equivalent in the sense of Theorem~\ref{thm:equi-cpo}. 
\end{theorem}

\begin{theorem} \label{thm:random-dpo-cpo}
\textnormal{\ref{eq:RA-DPO-one-seed}} and \textnormal{\ref{eq:RA-CPO1-Xi}} are equivalent in the sense of Theorem~\ref{thm:equi-dpo}. 
\end{theorem}

Proposition~\ref{thm:random-optimal-set} generalizes Proposition~\ref{thm:optimal-set}, whose proof is in Appendix~\ref{app:subsec_optimal_set}. 

\begin{prop}\label{thm:random-optimal-set}
  For a fixed feature $X=x$, a fixed seed $\Xi=\xi$ and a coverage value $t \in [0,1]$, among all the prediction sets $\{C(x,a,\xi)\}_{a\in \cA}$ that satisfy $\PP(Y(a)\in C(x,a,\xi)\given X=x) \ge t$ for every action $a \in \cA$, the following prediction sets has the largest risk averse utility $\nu_\RA(\cdot,\cdot;C)$ defined in~\eqref{eq:pi-nu-ra-def}:
  \[
  C(x,a,\xi)=\{y \in \cY \given u(a,y) \ge \gamma(x,t,a)\} ~~\text{for every action}~~ a \in \cA.
  \]
  Further, we have $\nu_\RA(x,\xi;C)=\theta(x,t)$.
\end{prop}

\subsection{Reformulation of randomized RA-CPO}
Like the discussion without randomization,~\ref{eq:RA-CPO2-Xi} is equivalent to the following optimization program:
\begin{equation}\label{eq:RA-CPO2'-Xi}\tag{RA-CPO-2'-Rand}
\begin{aligned}
\max_{t:\cX \times [0,1] \to [0,1]} &\EE_{X,\Xi}[\theta(X,t(X,\Xi))]\\
\text{s.t.} ~~&\EE_{X,\Xi}[t(X,\Xi)]\ge 1-\alpha 
\end{aligned}
\end{equation} 
In particular, if $t^*(x,\xi)$ is an optimal solution to~\ref{eq:RA-CPO2'-Xi}, then an optimal solution to~\ref{eq:RA-CPO2-Xi} is given by 
\$
C^*(x,a,\xi)=\{y \in \cY \given u(a,y) \ge \gamma (x,t^*(x,\xi),a)\} ~~\text{for every action} ~~a \in \cA.
\$
Intuitively, in RA-CPO-2'-Rand, the function $t(x,\xi)$ allocates the conditional (worst-case) coverage lower bounds across the $(x,\xi)$ values so as to satisfy the coverage guarantee while maximizing the corresponding optimal objective $\theta(x,t(x,\xi))$.

For $\beta\ge 0$ and $x\in\mathcal X$, define the pointwise maximizer set
\$
S(x,\beta):=\Big\{s\in[0,1]\colon  \theta(x,s)+\beta s=\max_{t\in[0,1]}(\theta(x,t)+\beta t)\Big\},
\$
and its endpoints $t_{\low}(x,\beta):=\inf S(x,\beta)$ and $t_{\high}(x,\beta):=\sup S(x,\beta)$.
The proof of Theorem~\ref{thm:seed_unified} is in Appendix~\ref{app:subsec_proof_unified_opt}. 

\begin{theorem}\label{thm:seed_unified}
Let $X\sim P_X$  and let $\Xi\sim \textnormal{Unif}(0,1)$ be the independent exogenous random variable. 
Then there exist $\beta^*\ge 0$ and a constant $c\in[0,1]$ such that 
\$
t^*(x,\xi)=
\begin{cases}
t_{\low}(x,\beta^*), & \xi\le c,\\
t_{\high}(x,\beta^*), & \xi>c,
\end{cases}
\$
is an optimal solution to \textnormal{\ref{eq:RA-CPO2'-Xi}} and satisfies $\EE[t^*(X,\Xi)]= 1-\alpha$.
\end{theorem}

\section{Technical proofs}
\subsection{Proof of Theorem~\ref{thm:op-decision}}
\label{app:subsec_op_decision}

\begin{proof}[Proof of Theorem~\ref{thm:op-decision}]
Write $\pi_{\RA}(\cdot)=\pi_\RA(\cdot;C)$ and $\nu_\RA(\cdot)=\nu_\RA(\cdot;C)$ for convenience.  
To prove Theorem~\ref{thm:op-decision}, it suffices to show that $\inf_{P\in \cF(C)}\nu(\pi_{\RA},P) = \max_{\pi\in\Pi}\inf_{P\in \cF(C)}\nu(\pi,P)$. 
%
We first show that no policy can achieve a worst-case value larger than
\$
\inf_{x:C(x,\pi_{\RA}(x))\neq\emptyset}\nu_{\RA}(x).
\$
Let $S:=\{x\in\cX:C(x,\pi_{\RA}(x))\neq\emptyset\}$. Since the prediction sets satisfy the coverage constraint along $\pi_{\RA}$, the set $S$ is nonempty.

Fix an arbitrary policy $\pi\in\Pi$. For any $\delta>0$, choose $x_\delta\in S$ such that
\$
\nu_{\RA}(x_\delta)\le \inf_{x\in S}\nu_{\RA}(x)+\delta.
\$
Since $x_\delta\in S$, choose $\widetilde y\in C(x_\delta,\pi_{\RA}(x_\delta))$. For any $\varepsilon>0$, choose $y^*_{\pi,\varepsilon}\in C(x_\delta,\pi(x_\delta))$ such that
\$
u(\pi(x_\delta),y^*_{\pi,\varepsilon})\le \inf_{y\in C(x_\delta,\pi(x_\delta))}u(\pi(x_\delta),y)+\varepsilon,
\$
whenever $C(x_\delta,\pi(x_\delta))\neq\emptyset$. If $C(x_\delta,\pi(x_\delta))=\emptyset$, choose $y^*_{\pi,\varepsilon}$ arbitrarily; under our convention $\inf_{y\in\emptyset}u(\pi(x_\delta),y)=u_{\max}$, the same inequality holds since $u\le u_{\max}$.

Define $\Omega^*_{\pi,\delta,\varepsilon}$ by placing all mass on $X=x_\delta$, setting $Y(\pi(x_\delta))=y^*_{\pi,\varepsilon}$, setting $Y(\pi_{\RA}(x_\delta))=\widetilde y$, and choosing the remaining potential outcomes arbitrarily. When $\pi(x_\delta)=\pi_{\RA}(x_\delta)$, take $\widetilde y=y^*_{\pi,\varepsilon}$. Then
\$
\Omega^*_{\pi,\delta,\varepsilon}\bigl(Y(\pi_{\RA}(X))\in C(X,\pi_{\RA}(X))\bigr)=1,
\$
so $\Omega^*_{\pi,\delta,\varepsilon}\in\cF(C)$. Under this distribution,
\$
\begin{aligned}
\nu(\pi,\Omega^*_{\pi,\delta,\varepsilon})&=u(\pi(x_\delta),y^*_{\pi,\varepsilon})\\
&\le \inf_{y\in C(x_\delta,\pi(x_\delta))}u(\pi(x_\delta),y)+\varepsilon\\
&\le \max_{a\in\cA}\inf_{y\in C(x_\delta,a)}u(a,y)+\varepsilon\\
&=\nu_{\RA}(x_\delta)+\varepsilon\\
&\le \inf_{x\in S}\nu_{\RA}(x)+\delta+\varepsilon.
\end{aligned}
\$
Therefore,
\$
\inf_{P\in\cF(C)}\nu(\pi,P)\le \inf_{x\in S}\nu_{\RA}(x)+\delta+\varepsilon.
\$
Letting $\delta\downarrow0$ and $\varepsilon\downarrow0$, we obtain
\$
\inf_{P\in\cF(C)}\nu(\pi,P)\le \inf_{x:C(x,\pi_{\RA}(x))\neq\emptyset}\nu_{\RA}(x).
\$
Since $\pi\in\Pi$ was arbitrary, taking the maximum over $\pi\in\Pi$ gives
\@\label{eq:bd1}
\max_{\pi\in\Pi}\inf_{P\in\cF(C)}\nu(\pi,P)\le \inf_{x:C(x,\pi_{\RA}(x))\neq\emptyset}\nu_{\RA}(x).
\@

In the next, we proceed to prove that the utility of $\pi_{\RA}$ further upper bounds the right-handed side of~\eqref{eq:bd1}. 
Consider any distribution $P \in \cF(C)$, which obeys $P(Y(\pi_{\RA}(X)) \in C(X,\pi_{\RA}(X))) \ge 1-\alpha$.

Note that on the event  $Y(\pi_\RA(X)) \in C(X,\pi_\RA(X))$, we have 
\$
u(\pi_\RA,Y(\pi_\RA(X))) \ge \inf_{y \in C(X,\pi_\RA(X))} u(\pi_\RA(X),y)=\max_{a \in \cA} \inf_{y \in C(X,a)} u(a,y) 
=\nu_\RA(X).
\$
The coverage guarantee 
$P(Y(\pi_\RA(X)) \in C(X,\pi_\RA(X))) \ge 1-\alpha$ 
thus implies 
\$
P(u(\pi_\RA(X),Y(\pi_\RA(X)))\ge \nu_\RA(X)) \ge 1-\alpha.
\$
As such, $\nu_{\RA}(\cdot)$ is a feasible utility certificate for the policy $\pi_{\RA}$ under $P$. Therefore, by the definition of $\nu(\pi_{\RA},P)$ as the maximum expected value over all feasible utility certificates, we have
\$
\nu(\pi_\RA, P)\geq \E_X[\nu_\RA(X)] \ge \inf_{x:C(x,\pi_\RA(x))\neq \emptyset}\nu_\RA(x).
\$
Now, by the arbitrariness of  $P \in \cF( C)$, we have
\@\label{eq:bd2}
\inf_{P \in \cF(C)}\nu(\pi_\RA,P)\ge \inf_{x:C(x,\pi_\RA(x))\neq \emptyset}\nu_\RA(x).
\@
Combining~\eqref{eq:bd1} and~\eqref{eq:bd2}, we obtain the desired result. 
\end{proof}

\subsection{Proof of Theorem \ref{thm:equi-cpo} and \ref{thm:random-cpo}}
\label{app:subsec_equiv_cpo}

The same construction proves both the deterministic and randomized statements. We therefore present only the randomized version.

\begin{proof}[Proof of Theorem~\ref{thm:random-cpo}]
    
We use the function $\cT(\cdot)$ to denote the risk-averse objective of  prediction sets, i.e., $\cT(C) =   \EE_{X,\Xi}[\nu_\RA(X,\Xi;C)]$ for any $\{C(x,a;\xi)\}_{a\in \cA}$.

\vspace{-0.5em}
\paragraph{Step 1: 
From randomized RA-CPO-1 to randomized RA-CPO-2.}
We first show that for any optimal solution to randomized RA-CPO-1, there exists a feasible solution to randomized RA-CPO-2 with identical objective value. 

Let $\{C^*_1(x,a,\xi)\}_{a\in \cA}$ be an optimal solution to~\ref{eq:RA-CPO1-Xi}. For each feature $x$ and random seed $\xi$, we denote the conditional coverage under the policy $\pi_{\RA}(x,\xi;C_1^*)$   as
\$
\beta(x,\xi) = \PP \big(Y(\pi_\RA(X,\Xi;C_1^*)) \in C_1^*(X,\pi_\RA(X,\Xi;C_1^*),\Xi)\given X=x,\Xi=\xi \big) .
\$
The feasibility of $C_1^*$ implies   $\EE[\beta(X,\Xi)]\geq 1-\alpha$. 
We now define the prediction set 
\$
C_2(x,a,\xi) = \begin{cases}
    C_1^*(x,a,\xi),\quad & \text{if}~~\PP(Y(a)\in C_1^*(x,a,\xi)\given X=x,\Xi=\xi) \geq \beta(x,\xi),\\ 
    \cY, \quad &  \text{if}~~\PP(Y(a)\in C_1^*(x,a,\xi)\given X=x,\Xi=\xi) < \beta(x,\xi).
\end{cases}
\$
It is clear that for any policy $\pi\in \Pi$, we must have
\$
\PP(Y(\pi(X,\Xi))\in C_2(X,\pi(X,\Xi),\Xi))
\geq \EE[\beta(X,\Xi)]\geq 1-\alpha.
\$
Therefore, $\{C_2(x,a,\xi)\}_{a\in \cA}$ is feasible for randomized RA-CPO-2.  

On the other hand, by construction, we know $C_2(x,a,\xi)=C_1^*(x,a,\xi)$ for $a = \pi_{\RA}(x,\xi;C_1^*)$. Therefore,
\$
\nu_{\RA}(x,\xi;C_2) \geq \nu_{\RA}(x,\xi;C_1^*).
\$
Moreover, for any $a\in\cA$, the construction of $C_2$ either keeps $C_1^*(x,a,\xi)$ unchanged or replaces it by $\cY$. In the latter case,
\$
\inf_{y\in \cY}u(a,y)
\leq
\inf_{y\in C_1^*(x,a,\xi)}u(a,y).
\$
Thus the worst-case utility associated with each action cannot increase when passing from $C_1^*$ to $C_2$, and hence
\$
\nu_{\RA}(x,\xi;C_2) \leq \nu_{\RA}(x,\xi;C_1^*).
\$
Combining the two inequalities, we have
\$
\nu_{\RA}(x,\xi;C_2) = \nu_{\RA}(x,\xi;C_1^*)
\$
for any value of $(x,\xi)$. This further implies $\cT(C_2) = \cT(C_1^*)$ and completes the proof of Step 1. 

\paragraph{Step 2.} 
Let $\OPT_1$ and $\OPT_2$ denote the optimal objective values of randomized RA-CPO-1 and RA-CPO-2, respectively. Note that the feasibility set of randomized RA-CPO-1 is a superset of that of randomized RA-CPO-2 since the constraint is weaker. As such,
\$
\OPT_1\geq \OPT_2.
\$
On the other hand, Step 1 shows that for any optimal solution $C_1^*$ to randomized RA-CPO-1, we can construct a feasible solution $C_2$ to randomized RA-CPO-2 such that
\$
\cT(C_2)=\cT(C_1^*)=\OPT_1.
\$
Therefore, $\OPT_2\geq \OPT_1$. Combining the two inequalities gives $\OPT_1=\OPT_2$. Hence, the constructed $C_2$ is an optimal solution to randomized RA-CPO-2.

Moreover, for any optimal solution $C_2^*(x,a,\xi)$ to randomized RA-CPO-2, since it is also feasible for randomized RA-CPO-1 and
\$
\cT(C_2^*)=\OPT_2=\OPT_1,
\$
it must also be an optimal solution to randomized RA-CPO-1.

\end{proof}

\subsection{Proof of Theorem~\ref{thm:equi-dpo} and \ref{thm:random-dpo-cpo}}
\label{app:subsec_dpo_cpo}

The deterministic proof follows the same argument as the randomized proof. For this reason, we present the randomized version only, since the deterministic formulation is recovered by taking $\pi(x,\xi)=\pi(x)$ and $\nu(x,\xi)=\nu(x)$.

\begin{proof}[Proof of Theorem~\ref{thm:random-dpo-cpo}] We show the equivalence in two steps.

\vspace{-0.5em}
\paragraph{Step 1: From randomized RA-DPO to randomized RA-CPO-1.} 
In this part, we aim to show that for any optimal solution $(\pi^*,\nu^*)$ to  randomized RA-DPO, there exists prediction sets $\{C_1^*(x,a,\xi)\}_{a\in \cA}$ feasible in randomized RA-CPO-1 with a no smaller objective value. 

Consider any  feasible solution $\pi^*(x,\xi)$, $\nu^*(x,\xi)$ to~\ref{eq:RA-DPO-one-seed}, which obeys  \$
\PP_{X,\Xi}(u(\pi^*(X,\Xi),Y(\pi^*(X,\Xi)))\ge \nu^*(X,\Xi))\ge 1-\alpha,
\$ 
and maximizes the target function $\E_{X,\Xi}[\nu^*(X,\Xi)]$. Then we build the prediction set 
\$
C_1^*(x,a,\xi) = \begin{cases}
    \{y\in \cY \given u(\pi^*(x,\xi),y)\ge \nu^*(x,\xi)\},\quad &a = \pi^*(x,\xi),\\ 
      \cY,\quad &a\neq \pi^*(x,\xi).
\end{cases}
\$ Since $C_1^*(x,a,\xi)=\cY$ when $a \neq \pi^*(x,\xi)$, we have
\$
\PP\big(Y(\pi_\RA(X,\Xi;C_1^*))\in C_1^*(X,\pi_\RA(X,\Xi;C_1^*),\Xi) \big)&\ge \PP\big(Y(\pi^*(X,\Xi))\in C_1^*(X,\pi^*(X,\Xi),\Xi) \big)\\
&= \PP(u(\pi^*(X,\Xi),Y(\pi^*(X,\Xi)))\ge \nu^*(X,\Xi))\\
&\ge 1-\alpha,
\$
so $C_1^*$ is feasible solution to randomized RA-CPO-1. 
On the other hand, note that for any $(x,\xi)$, by definition  
\$
\max_{a\in \cA}\inf_{y \in C_1^*(x,a,\xi)}u(a,y) \ge \inf_{y \in C_1^*(x,\pi^*(x,\xi),\xi)}u(\pi^*(x,\xi),y)\ge \nu^*(x,\xi).
\$
Taking  expectation of both sides, we have 
\$
\E[\nu^*(X,\Xi)]
\leq \EE\Big[ \inf_{y \in C_1^*(X,\pi^*(X,\Xi),\Xi)}u(\pi^*(X,\Xi),y)  \Big]
\leq \EE \Big[\max_{a\in \cA}\inf_{y\in C_1^*(X,a,\Xi)}u(a,y) \Big].
\$

\paragraph{Step 2: From Randomized RA-CPO-1 to randomized RA-DPO.}
In this part, we show that for any optimal solution $\{C_1^*(x,a,\xi)\}_{a\in \cA}$ to randomized RA-CPO-1, there exists a feasible solution $(\pi^*,\nu^*)$ to randomized RA-DPO with a no smaller objective value.  

Consider any optimal solution $\{C_1^*(x,a,\xi)\}_{a\in \cA}$ to~\ref{eq:RA-CPO1-Xi}. 
Then it must obey  
\$
\PP(Y(\pi_\RA(X,\Xi;C_1^*)) \in C_1^*(X,\pi_\RA(X,\Xi;C_1^*),\Xi))\ge 1-\alpha,
\$ 
and it maximizes the objective 
$ 
\EE[\max_{a\in \cA}\inf_{y\in C(X,a,\Xi)}u(a,y)]
$. 
Define the policy and utility certificate 
\$
\pi^*(x,\xi)=\pi_\RA(x,\xi;C_1^*), \quad \nu^*(x,\xi)=\nu_\RA(x,\xi;C_1^*).
\$
Note that by definition,  $Y(\pi_\RA(X,\Xi;C_1^*))\in C_1^*(X,\pi_\RA(X,\Xi;C_1^*),\Xi)$ implies 
\$
u(\pi^*(X,\Xi),Y(\pi^*(X,\Xi)))\ge \inf_{y \in C_1^*(X,\pi^*(X,\Xi),\Xi)}u(\pi^*(X,\Xi),y)=\nu^*(X,\Xi).
\$
This further implies 
\$
\PP\big(u(\pi^*(X,\Xi),Y(\pi^*(X,\Xi))) \ge \nu^*(X,\Xi)\big)\ge \PP\big(Y(\pi_\RA(X,\Xi;C_1^*))\in C_1^*(X,\pi_\RA(X,\Xi;C_1^*),\Xi)\big)\ge 1-\alpha.
\$
As such, we know 
$(\pi^*,\nu^*)$  are feasible solutions to randomized RA-DPO.

On the other hand, by definition we know $\E[\nu^*(X,\Xi)]=\E[\max_{a\in \cA}\min_{y\in C_1^*(X,a,\Xi)}u(a,y)]$, the (optimal) objective value of~\ref{eq:RA-CPO1-Xi}, thereby completing the second part. 

Finally, combining the two parts, we know that \ref{eq:RA-DPO-one-seed} and~\ref{eq:RA-CPO1-Xi} are equivalent. 
\end{proof}

\subsection{Proof of Property~\ref{thm:right-con}}
\label{app:subsec_right_continuity}

We show that the upper quantile defined in~\eqref{eq:def_upp_qt} is right continuous.


\begin{property}\label{thm:right-con}
The function $q(\alpha)= \sup \{z \in \RR \given \PP(Z\le z) \le \alpha\}$ is right-continuous in $\alpha\in [0,1)$.
\end{property}

\begin{proof}[Proof of property \ref{thm:right-con}]
    
For $\alpha\in[0,1)$ let
\$
A_\alpha := \{ z\in\mathbb{R} : F(z)\le \alpha\},
\$
so that $q(\alpha)=\sup A_\alpha$.

If $\alpha\le \beta$, then $A_\alpha\subseteq A_\beta$, hence
\$
q(\alpha)=\sup A_\alpha \le \sup A_\beta = q(\beta).
\$
Therefore $q$ is nondecreasing.

Fix $\alpha\in[0,1)$ and let $\alpha_n \downarrow \alpha$. By monotonicity, the sequence
$q(\alpha_n)$ is nonincreasing, hence the limit
\$
\beta := \lim_{n\to\infty} q(\alpha_n) = \inf_{n\ge 1} q(\alpha_n)
\$
exists in $[-\infty,\infty]$. We claim that $\beta=q(\alpha)$.

First, since $\alpha_n\ge \alpha$ we have $A_\alpha\subseteq A_{\alpha_n}$ for each $n$, and thus $\forall n$,
\$
q(\alpha)=\sup A_\alpha \le \sup A_{\alpha_n}=q(\alpha_n)
\$
Taking the infimum over $n$ yields $q(\alpha)\le \beta$.

To prove the reverse inequality, assume for contradiction that $\beta>q(\alpha)$.
Choose any $z$ such that
\$
q(\alpha) < z < \beta.
\$
Because $q(\alpha_n)\downarrow \beta$, we have $q(\alpha_n)\ge \beta>z$ for all $n$.
By the definition of supremum, for each $n$ there exists $y_n\in A_{\alpha_n}$ such that
\$
z<y_n\le q(\alpha_n)\quad\text{and}\quad F(y_n)\le \alpha_n.
\$
Since $F$ is nondecreasing and $z<y_n$, it follows that $\forall  n$,
\$
F(z)\le F(y_n)\le \alpha_n 
\$
Letting $n\to\infty$ and using $\alpha_n\downarrow \alpha$ gives $F(z)\le \alpha$, i.e., $z\in A_\alpha$.
This contradicts $z> \sup A_\alpha = q(\alpha)$. Hence $\beta\le q(\alpha)$.

Combining both inequalities, $\beta=q(\alpha)$, and therefore $q(\alpha_n)\to q(\alpha)$ whenever
$\alpha_n \downarrow \alpha$. This proves that $q$ is right-continuous on $[0,1)$.
\end{proof}

\subsection{Proof of Proposition \ref{thm:optimal-set} and \ref{thm:random-optimal-set} }
\label{app:subsec_optimal_set}

Since the randomized problem is more general, we prove Proposition~\ref{thm:random-optimal-set} only.

\begin{proof}[Proof of Proposition~\ref{thm:random-optimal-set}]
Consider any suite of prediction sets $\{C(x,a,\xi)\}_{a \in \cA}$ (with $x$ and $\xi$ fixed) that satisfy the coverage constraints  
\$
\PP(Y(a)\in C(X,a,\Xi) \given X=x,\Xi=\xi)\ge t,\quad \forall a\in \cA.
\$
This coverage guarantee implies 
\$
\PP\Big(u(a,Y(a))\ge \inf_{y \in C(X,a,\Xi)}u(a,y)\Biggiven X=x, \Xi=\xi \Big)\ge t
\$
Then for any constant $\epsilon>0$, we know 
\$
& \PP\Big(u(a,Y(a))\le \inf_{y \in C(X,a,\Xi)}u(a,y)-\epsilon \Biggiven X=x,\Xi=\xi \Big) \\ 
&\le \PP\Big(u(a,Y(a)) < \inf_{y \in C(X,a,\Xi)}u(a,y) \Biggiven X=x,\Xi=\xi\Big) \le 1-t,
\$
which implies 
\$
\inf_{y \in C(x,a,\xi)}u(a,y) \le \uquant_{1-t}[u(a,Y(a))|X=x]
\$
by the arbitrariness of $\epsilon>0$. Taking maximum over $a\in \cA$ on both sides yields
\[
\max_{a \in \cA} \inf_{y \in C(x,a,\xi)}u(a,y) \le \max_{a \in \cA}~\uquant_{1-t}[u(a,Y(a))|X=x]=\theta(x,t),
\]
which means $\theta(x,t)$ is an upper bound of $\nu_{RA}(x,\xi;C)$ for any prediction sets obeying the coverage constraints.

We now define the prediction sets
\$
C^*(x,a,\xi)=\{y \in \cY \given u(a,y) \ge \gamma(x,t,a)\}, ~~\text{for each}~  a \in \cA.
\$
We first show that these sets   satisfy the coverage guarantee. For any constant $\epsilon>0$, we note that 
\$
&\PP\big(u(a,Y(a)) \le \uquant_{1-t}[u(a,Y(a)) \given X=x]-\epsilon \biggiven X=x\big) \\
&=\PP\big(u(a,Y(a)) \le \sup\{z \in \RR \given \PP(u(a,Y(a))\le z \given X=x)\le 1-t\}-\epsilon \biggiven X=x\big) \leq 1-t. 
\$
The arbitrariness of $\epsilon>0$ implies 
\$
\PP(u(a,Y(a)) < \uquant_{1-t}[u(a,Y(a)) \given X=x] \given X=x )\le 1-t.
\$
Equivalently, we have the coverage guarantee: for each $a\in \cA$, 
\$
\PP(Y(a) \in C^*(x,a,\xi)\given X=x) = 
&\PP(u(a,Y(a)) \ge \uquant_{1-t}[u(a,Y(a)) \given X=x] \given X=x )\geq t.
\$
On the other hand, by construction we have 
\$
\nu_\RA(x,\xi;C^*)=\max_{a \in \cA} \inf_{y \in C^*(x,a,\xi)} u(a,y) \ge \max_{a \in \cA}  \gamma(x,t,a)=\theta(x,t).
\$
Since $\theta(x,t)$ is an upper bound for $\nu_{\RA}(x,\xi;C)$ for any prediction sets $\{C(x,a,\xi)\}_{a\in \cA}$ obeying the coverage guarantees, we know the above $\{C^*(x,a,\xi)\}_{a\in \cA}$ attain the optimal utility. 
\end{proof}

\subsection{Proof of Theorem \ref{thm:unified} and \ref{thm:seed_unified}}
\label{app:subsec_proof_unified_opt}

Here we prove the more general randomized version, Theorem~\ref{thm:seed_unified}. The only difference is that  Theorem~\ref{thm:unified} additionally assumes the maximizer is unique, thus $\inf S(x,\beta) = \sup S(x,\beta)$ for $P_X$-a.s.~$x\in \cX$. Theorem~\ref{thm:unified} can thus be proved by simplifying the proof below with $t_{\low}(x,\beta)=t_{\high}(x,\beta)$.

Intuitively, in RA-CPO-2'-Rand, the function $t(x,\xi)$ allocates the conditional (worst-case) coverage  across the $(x,\xi)$ values so as to satisfy the coverage guarantee while maximizing the corresponding optimal objective $\theta(x,t(x,\xi))$.  We shall show that the given solution $t^*(x,\xi)$ is indeed optimal.

\begin{proof}[Proof of Theorem~\ref{thm:seed_unified}]
Fix $\beta\ge0$ and define
\$
m_\beta(x):=\sup_{s\in[0,1]}\{\theta(x,s)+\beta s\}.
\$
For each $x\in\cX$, define the pointwise maximizer set
\$
S(x,\beta):=\Big\{s\in[0,1]\given \theta(x,s)+\beta s=m_\beta(x)\Big\},
\$
and its endpoints
\$
t_{\low}(x,\beta):=\inf S(x,\beta),\qquad t_{\high}(x,\beta):=\sup S(x,\beta).
\$
Since $s\mapsto\theta(x,s)$ is left-continuous and nonincreasing on $[0,1]$, the set $S(x,\beta)$ is nonempty and compact,
so $t_{\low}(x,\beta)$ and $t_{\high}(x,\beta)$ are well-defined. Moreover, $x\mapsto m_\beta(x)$ is measurable, and the endpoint mapping 
$x\mapsto t_{\low}(x,\beta)$ and $x\mapsto t_{\high}(x,\beta)$ can be chosen to be measurable. 

We denote the optimal objective of~\ref{eq:RA-CPO2'-Xi} as $\OPT_2$. In the following, we study the subgradients of the objective functions and establish the form of the optimal prediction sets.

\paragraph{Step 1: Weak duality upper bound for the randomized primal.}
Let $T=t(X,\Xi)\in[0,1]$ be any feasible randomized rule, i.e. $\EE[T]\ge 1-\alpha$.
For every value of $(x,s)$ we have
\$
m_\beta(x)=\max_{u\in[0,1]}\{\theta(x,u)+\beta u\} \ge \theta(x,s)+\beta s.
\$
Plugging in  $s=T$ and taking expectations on both sides, we have 
\$
\EE[\theta(X,T)]  \le \EE[m_\beta(X)]-\beta\EE[T]
\le \EE[m_\beta(X)]-\beta(1-\alpha)  = D(\beta),
\$
where we define 
\$
D(\beta):=\EE[m_\beta(X)]-\beta(1-\alpha),\qquad \beta\ge0.
\$
Therefore,
\@\label{eq:weak_duality_unified}
\OPT_2 \le \inf_{\beta\ge0} D(\beta).
\@

\paragraph{Step 2: Existence of a dual minimizer.}
For each fixed $x\in \cX$, the map $\beta\mapsto m_\beta(x)$ is convex because
\$
m_\beta(x)=\max_{s\in[0,1]}\{\theta(x,s)+\beta s\}
\$
is a pointwise maximum of affine functions of $\beta$. Hence the function $D(\cdot)$ is convex on $[0,\infty)$.
Moreover, since $\theta(x,s)$ is bounded, there exists $C<\infty$ such that $|\theta(x,s)|\le C$ for all $(x,s)$.
Therefore, for any $x\in \cX$ and $\beta\ge0$, we have 
\$
m_\beta(x) \ge \theta(x,1)+\beta\cdot 1 \ge -C+\beta,
\$
and consequently
\$
D(\beta)=\EE[m_\beta(X)]-\beta(1-\alpha) \ge (-C+\beta)-\beta(1-\alpha)= -C+\alpha\beta \xrightarrow[\beta\to\infty]{}+\infty.
\$
This implies $D$ is coercive on $[0,\infty)$ and attains its minimum at some $\beta^*\ge0$.

\paragraph{Step 3: Subgradient of $D$ and the attainable interval at $\beta^*$.} This step consists of three parts: (3a) derive the pointwise subgradient of the mapping $\beta\mapsto m_\beta(x)$, (3b)  derive the subgradient of the mapping $\beta\mapsto \EE[m_\beta(X)]$, and (3c) subdifferential of the mapping $\beta\mapsto D(\beta)$. 

\smallskip
\noindent\underline{\emph{(3a) Pointwise subgradient of $m_\beta(x)$ on $[0,\infty)$.}}
We first show a pointwise Lipschitz bound in $\beta$. For every $\beta,\beta'\ge0$,
\$
m_{\beta'}(x)=\max_{s\in[0,1]}\{\theta(x,s)+\beta s+(\beta'-\beta)s\}\le m_\beta(x)+|\beta'-\beta|,
\$
and symmetrically $m_\beta(x)\le m_{\beta'}(x)+|\beta'-\beta|$. Hence
\@\label{eq:mbeta_lipschitz_pointwise}
|m_{\beta'}(x)-m_\beta(x)|\le |\beta'-\beta|.
\@
In particular, $\beta\mapsto m_\beta(x)$ is continuous on $[0,\infty)$.

We now proceed to study the subgradients of $\beta \mapsto m_\beta(x)$. 
By standard convex analysis results, a convex function $f$ on $[0,\infty)$, for every $\beta>0$, the subgradients are given by 
\$
\partial f(\beta)=[f'_-(\beta),f'_+(\beta)],
\$
and at the boundary, the subgradients are given by 
\$
\partial f(0)=(-\infty,f'_+(0)].
\$
We apply this with $f(\beta)=m_\beta(x)$, and 
calculate the one-sided derivatives. 

\smallskip
\noindent\textbf{Right derivative.}
For $h>0$ define
\$
q_h^+(x):=\frac{m_{\beta+h}(x)-m_\beta(x)}{h}.
\$
For any $s\in S(x,\beta)$,
\$
m_{\beta+h}(x)\ge \theta(x,s)+(\beta+h)s=m_\beta(x)+hs,
\$
so we have the lower bound 
\@\label{eq:lb_qh}
q_h^+(x)\ge \max_{s\in S(x,\beta)}s=t_{\high}(x,\beta).
\@

On the other hand, for each $h>0$,  pick any element $s_h\in S(x,\beta+h)$. Then we have 
\$
m_{\beta+h}(x)=\theta(x,s_h)+\beta s_h+hs_h\le m_\beta(x)+hs_h,
\$
because $m_\beta(x)\ge \theta(x,s_h)+\beta s_h$. Hence $q_h^+(x)\le s_h\le 1$ for any $h>0$. 

Now take any sequence $h_k\downarrow 0$. By compactness of $[0,1]$, along a subsequence (still denoted)
$s_{h_k}\to \bar s\in[0,1]$. By \eqref{eq:mbeta_lipschitz_pointwise}, we have  $m_{\beta+h_k}(x)\to m_\beta(x)$, and clearly
$h_k s_{h_k}\to 0$, so
\$
\theta(x,s_{h_k})+\beta s_{h_k}
= m_{\beta+h_k}(x)-h_k s_{h_k} \longrightarrow m_\beta(x).
\$
Since $s\mapsto\theta(x,s)$ is nonincreasing and left continuous, we know 
\$
\limsup_{k\to\infty}\theta(x,s_{h_k})\le \theta(x,\bar s).
\$
Indeed, if $s_{h_k}\ge \bar s$ along a subsequence then monotonicity gives
$\theta(x,s_{h_k})\le \theta(x,\bar s)$ there; if $s_{h_k}<\bar s$ along a subsequence, then after taking a further subsequence we may assume
$s_{h_k}\uparrow\bar s$ and left continuity yields $\theta(x,s_{h_k})\to\theta(x,\bar s)$.
Combining the two cases gives the desired inequality above. Since $\beta s_{h_k}\to \beta\bar s$,
\$
m_\beta(x)=\lim_{k\to \infty} \big(\theta(x,s_{h_k})+\beta s_{h_k}\big)\le \theta(x,\bar s)+\beta\bar s ,
\$
which means $\bar s\in S(x,\beta)$ is also a maximizer. Hence $\limsup_{k\to \infty} s_{h_k}\le \sup S(x,\beta)=t_{\high}(x,\beta)$.
Now, since we have shown that $q_{h_k}^+(x)\le s_{h_k}$, we get $\limsup_{k\to \infty} q_{h_k}^+(x)\le t_{\high}(x,\beta)$.
Together with $q_h^+(x)\ge t_{\high}(x,\beta)$ in~\eqref{eq:lb_qh}, we know $q_{h_k}^+(x)\to t_{\high}(x,\beta)$ for any sequence $h_k\downarrow 0$, i.e.
\$
(m_\beta(x))'_+(\beta)=t_{\high}(x,\beta),\qquad \forall \beta\ge0.
\$

\smallskip
\noindent\textbf{Left derivative (interior points only).}
For any $\beta>0$ and $h\in(0,\beta)$ we define
\$
q_h^-(x):=\frac{m_\beta(x)-m_{\beta-h}(x)}{h}.
\$

First, for any $s\in S(x,\beta)$, we have 
\$
m_{\beta-h}(x)\ge \theta(x,s)+(\beta-h)s=m_\beta(x)-hs,
\$
so $q_h^-(x)\le s$ for all $s\in S(x,\beta)$, hence we have the upper bound $q_h^-(x)\le t_{\low}(x,\beta)$.

On the other hand, for any $h>0$, pick any element $r_h\in S(x,\beta-h)$. Then
\$
m_\beta(x)\ge \theta(x,r_h)+\beta r_h=m_{\beta-h}(x)+hr_h,
\$
so we have $q_h^-(x)\ge r_h\ge 0$.

Now take any sequence $h_k\downarrow 0$ and (by compactness) extract a subsequence such that $r_{h_k}\to \bar r\in[0,1]$.
Using \eqref{eq:mbeta_lipschitz_pointwise}, we have $m_{\beta-h_k}(x)\to m_\beta(x)$ and $h_k r_{h_k}\to 0$ as $k\to \infty$, hence
\$
\theta(x,r_{h_k})+\beta r_{h_k}=m_{\beta-h_k}(x)+h_k r_{h_k} \longrightarrow m_\beta(x).
\$
The same left-continuity and monotonicity argument as before then implies $\bar r\in S(x,\beta)$.
Since $r_{h_k}\le q_{h_k}^-(x)\le t_{\low}(x,\beta)$ for all $k$, taking $k\to \infty$,  we get $\bar r\le t_{\low}(x,\beta)$.
Meanwhile, as  $\bar r\in S(x,\beta)$ also implies $\bar r\ge \inf S(x,\beta)=t_{\low}(x,\beta)$, we know  $\bar r=t_{\low}(x,\beta)$.
Therefore $r_{h_k}\to t_{\low}(x,\beta)$, and by the squeeze $r_{h_k}\le q_{h_k}^-(x)\le t_{\low}(x,\beta)$ we conclude
$q_h^-(x)\to t_{\low}(x,\beta)$ as $h\downarrow 0$, i.e.
\$
(m_\beta(x))'_-(\beta)=t_{\low}(x,\beta),\qquad \forall \beta>0.
\$

\smallskip
\noindent\textbf{Putting two derivatives together.} Putting the above two parts together, we know that for $\beta>0$, the subgradient  of $\beta\mapsto m_\beta(x)$ is given by 
\begin{equation}\label{eq:subgrad_mbeta_pointwise_unified_S}
\partial_\beta m_\beta(x)=\big[(m_\beta(x))'_-(\beta),(m_\beta(x))'_+(\beta)\big]
=\big[t_{\low}(x,\beta),t_{\high}(x,\beta)\big],
\end{equation}
and at $\beta=0$,
\$
\partial_\beta m_0(x)=(-\infty,(m_0(x))'_+(0)]=(-\infty,t_{\high}(x,0)].
\$

\smallskip
\noindent\underline{\emph{(3b) Subgradient of $F(\beta):=\EE[m_\beta(X)]$.}}
Fix any $\beta\ge0$ and define
\$
F(\beta):=\EE[m_\beta(X)].
\$
By \eqref{eq:mbeta_lipschitz_pointwise}, for every $h>0$, we have 
\$
0\le \frac{m_{\beta+h}(x)-m_\beta(x)}{h}\le 1,
\qquad
\text{and }\quad  0\le \frac{m_\beta(x)-m_{\beta-h}(x)}{h}\le 1~~\text{for}~h\in(0,\beta) ~~ \text{when} ~~\beta>0 .
\$
Moreover, the pointwise limits of these quantities have been studied in (3a), namely, 
for $\beta\ge0$,
\$
\frac{m_{\beta+h}(x)-m_\beta(x)}{h} \xrightarrow[h\downarrow0]{} t_{\high}(x,\beta),
\$
and for $\beta>0$,
\$
\frac{m_\beta(x)-m_{\beta-h}(x)}{h} \xrightarrow[h\downarrow0]{} t_{\low}(x,\beta).
\$
Therefore, by dominated convergence theorem,
\$
F'_+(\beta)
=\lim_{h\downarrow 0}\frac{F(\beta+h)-F(\beta)}{h}
=\EE\!\left[\lim_{h\downarrow0}\frac{m_{\beta+h}(X)-m_\beta(X)}{h}\right]
=\EE[t_{\high}(X,\beta)],
\$
and for $\beta>0$,
\$
F'_-(\beta)
=\lim_{h\downarrow 0}\frac{F(\beta)-F(\beta-h)}{h}
=\EE\!\left[\lim_{h\downarrow0}\frac{m_\beta(X)-m_{\beta-h}(X)}{h}\right]
=\EE[t_{\low}(X,\beta)].
\$
Since $F$ is one-dimensional convex on $[0,\infty)$, for $\beta>0$ we have
\@\label{eq:subdiff_F_pos_S}
\partial F(\beta)=[F'_-(\beta),F'_+(\beta)]
=\big[\EE[t_{\low}(X,\beta)],\EE[t_{\high}(X,\beta)]\big],
\@
and at $\beta=0$ (boundary case),
\@\label{eq:subdiff_F_at0_S}
\partial F(0)=(-\infty,F'_+(0)]=(-\infty,\EE[t_{\high}(X,0)]].
\@

\smallskip
\noindent\underline{\emph{(3c) Subdifferential of $D$ and the attainable interval.}}
Recall $D(\beta)=F(\beta)-\beta(1-\alpha)$. Hence
\$
\partial D(\beta)=\partial F(\beta)-(1-\alpha).
\$
In particular, for $\beta>0$,
\@\label{eq:subdiff_D_pos_S}
\partial D(\beta)=\big[\EE[t_{\low}(X,\beta)]-(1-\alpha),\EE[t_{\high}(X,\beta)]-(1-\alpha)\big],
\@
while at $\beta=0$,
\@\label{eq:subdiff_D_at0_S}
\partial D(0)=(-\infty,\EE[t_{\high}(X,0)]-(1-\alpha)].
\@

Since $\beta^*$ minimizes the convex function $D$ over $[0,\infty)$, the 1D constrained optimality condition gives
\$
0\in \partial D(\beta^*)+N_{[0,\infty)}(\beta^*),
\$
where we define the normal cone $N_{[0,\infty)}(\beta^*)=\{0\}$ if $\beta^*>0$ and $N_{[0,\infty)}(0)=(-\infty,0]$.
Using \eqref{eq:subdiff_D_pos_S}--\eqref{eq:subdiff_D_at0_S}:
\begin{itemize}
\item If $\beta^*>0$, then $0\in\partial D(\beta^*)$, hence
\$
1-\alpha\in\big[\EE[t_{\low}(X,\beta^*)],\EE[t_{\high}(X,\beta^*)]\big].
\$
\item If $\beta^*=0$, then $\partial D(0)\cap[0,\infty)\neq\emptyset$, which is equivalent to
\$
\EE[t_{\high}(X,0)]\ge 1-\alpha.
\$
\end{itemize}

\paragraph{Step 4: Construct an optimal (randomized) solution at $\beta^*$.} We now construct an optimal solution that attains the optimal value $D(\beta^*)$. 

\smallskip
\noindent\underline{\emph{(4a) Lagrangian identity under pointwise optimality.}}
Fix $\beta\ge0$ and let $T=t(X,\Xi)$ satisfy $T\in S(X,\beta)$ a.s.
Then $\theta(X,T)+\beta T=m_\beta(X)$ a.s., hence
\@\label{eq:Lag_identity_unified_S}
\EE[\theta(X,T)+\beta T]=\EE[m_\beta(X)].
\@

\smallskip
\noindent\underline{\emph{(4b) Achievable means via seed-mixtures of endpoints.}}
Fix $\beta\ge0$ and write
\$
a_-(\beta):=\EE[t_{\low}(X,\beta)],\qquad a_+(\beta):=\EE[t_{\high}(X,\beta)].
\$
If $a_-(\beta)=a_+(\beta)$, then any $T$ with $T\in S(X,\beta)$ a.s.\ satisfies $\EE[T]=a_-(\beta)=a_+(\beta)$.
Otherwise, for any target mean $m\in[a_-(\beta),a_+(\beta)]$, define
\$
\lambda:=\frac{a_+(\beta)-m}{a_+(\beta)-a_-(\beta)}\in[0,1], \qquad
t_m(x,\xi):=
\begin{cases}
t_{\low}(x,\beta), & \xi\le \lambda,\\
t_{\high}(x,\beta), & \xi>\lambda.
\end{cases}
\$
Then $T_m:=t_m(X,\Xi)$ satisfies
$T_m\in\{t_{\low}(X,\beta),t_{\high}(X,\beta)\}\subseteq S(X,\beta)$ a.s.\ and $\EE[T_m]=m$.

\smallskip
\noindent\underline{\emph{(4c) Apply (4a)--(4b) at $\beta^*$ and prove optimality.}}
By Step~1 and since $\beta^*\in\argmin_{\beta\ge0}D(\beta)$,
\$
\textnormal{OPT}_2 \le \inf_{\beta\ge0}D(\beta)=D(\beta^*).
\$
We now construct a feasible $T^*$ that attains $D(\beta^*)$.

\smallskip
First, if $\beta^*>0$, then 
Step~3 gives $1-\alpha\in[a_-(\beta^*),a_+(\beta^*)]$, so (4b) provides a randomized solution $T^*$ with
$T^*\in S(X,\beta^*)$ a.s.\ and $\EE[T^*]=1-\alpha$.
Then by \eqref{eq:Lag_identity_unified_S}, we have 
\$
\EE[\theta(X,T^*)]
=\EE[m_{\beta^*}(X)]-\beta^*\EE[T^*]
=\EE[m_{\beta^*}(X)]-\beta^*(1-\alpha)
=D(\beta^*).
\$
Thus $\EE[\theta(X,T^*)]=\textnormal{OPT}_2$.

\smallskip
On the other hand, if  $\beta^*=0$, then 
Step~3 gives $\EE[t_{\high}(X,0)]\ge 1-\alpha$ and we have $t_{\low}(X,0)=0$ a.s., hence
$1-\alpha\in[a_-(0),a_+(0)]$ and step (4b) yields a randomized solution $T^*$ with $\EE[T^*]=1-\alpha$ and $T^*\in S(X,0)$ a.s.
Then $\theta(X,T^*)=m_0(X)$ a.s., so $\EE[\theta(X,T^*)]=\EE[m_0(X)]=D(0)=D(\beta^*)$. 
\end{proof}

\subsection{Proof of Theorem \ref{thm:equi-seed}} 
\label{app:subsec_equi_seed}
\begin{lemma}[Universal simulation from a single uniform seed] 
\label{lem:universal_simulation}
Let $\mu$ be any Borel probability measure on $[0,1]^2$. There exists a Borel measurable map
$T:[0,1]\to[0,1]^2$ such that if $\Xi\sim \mathrm{Unif}(0,1)$, then $T(\Xi)\sim \mu$.
Moreover, if $\Xi$ is independent of a random element $Z$, then $T(\Xi)$ is also independent of $Z$.
\end{lemma}

We use $P_1$ to denote the single-seed optimization program and $P_2$ to denote the two-seed optimization program, and their optimal objectives are $\OPT(P_1)$ and $\OPT(P_2)$, respectively. 

We prove the two inequalities $\OPT(P_1)\ge \OPT(P_2)$ and
$\OPT(P_2)\ge \OPT(P_1)$. Through the proof, we can easily see that from any optimal realization of $P_1$, we can construct an optimal realization of $P_2$ and from any optimal realization of $P_2$, we can construct an optimal realization of $P_1$. 

\noindent\textbf{Step 1: $\OPT(P_1)\ge \OPT(P_2)$.}
Let $(\nu_2,\pi_2,\Xi_1,\Xi_2)$ be any feasible realization for $P_2$.
Let $\mu$ denote the joint distribution of $(\Xi_1,\Xi_2)$ on $[0,1]^2$.
By Lemma~\ref{lem:universal_simulation}, there exists a Borel map
$T=(T_1,T_2):[0,1]\to[0,1]^2$ such that for $\Xi\sim\mathrm{Unif}(0,1)$,
\$
(T_1(\Xi),T_2(\Xi)) \stackrel{d}{=} (\Xi_1,\Xi_2).
\$
Since $\Xi\perp (X,\{Y(a)\}_{a\in\cA})$, Lemma~\ref{lem:universal_simulation} also gives
\$
(T_1(\Xi),T_2(\Xi)) \perp (X,\{Y(a)\}_{a\in\cA}).
\$
Define the single-seed decision rules
\$
\nu_1(x,\xi):=\nu_2(x,T_1(\xi))\qquad
\pi_1(x,\xi):=\pi_2(x,T_2(\xi)).
\$
Then the objective value is preserved:
\$
\EE[\nu_1(X,\Xi)]
=\EE[\nu_2(X,T_1(\Xi))]
=\EE[\nu_2(X,\Xi_1)].
\$
Likewise, the constraint probability is preserved by the distributional identity
$(T_1(\Xi),T_2(\Xi))\stackrel{d}{=}(\Xi_1,\Xi_2)$ together with independence from the data:
\$
&\PP(u(\pi_1(X,\Xi),Y(\pi_1(X,\Xi)))\ge \nu_1(X,\Xi))\\
=&\PP(u(\pi_2(X,T_2(\Xi)),Y(\pi_2(X,T_2(\Xi))))\ge \nu_2(X,T_1(\Xi)))\\
=&\PP(u(\pi_2(X,\Xi_2),Y(\pi_2(X,\Xi_2)))\ge \nu_2(X,\Xi_1))
~\ge~ 1-\alpha.
\$
Hence $(\nu_1,\pi_1,\Xi)$ is a feasible realization for $P_1$ and achieves the same objective value as
$(\nu_2,\pi_2,\Xi_1,\Xi_2)$ in $P_2$. Taking suprema over all feasible realizations $(\nu_2,\pi_2,\Xi_1,\Xi_2)$ yields 
$\OPT(P_1)\ge \OPT(P_2)$.

\noindent\textbf{Step 2: $\OPT(P_2)\ge \OPT(P_1)$.}
Let $(\nu_1,\pi_1,\Xi)$ be any feasible realization for $P_1$.
In the two-seed problem $P_2$, choose the admissible coupling
$\Xi_1=\Xi_2=\Xi$ where $\Xi\sim\mathrm{Unif}(0,1)$ and $\Xi\perp (X,\{Y(a)\}_{a\in\cA})$.
This construction is admissible: it satisfies the marginal uniformity requirements for both seeds, and the two-seed program does not require the seeds to be independent of each other.
Now define
\$
\nu_2(x,\xi_1):=\nu_1(x,\xi_1)\qquad
\pi_2(x,\xi_2):=\pi_1(x,\xi_2).
\$
Under $\Xi_1=\Xi_2=\Xi$, we have
\$
\EE[\nu_2(X,\Xi_1)]=\EE[\nu_1(X,\Xi)],
\$
and the constraint event coincides pointwise:
\$
u(\pi_2(X,\Xi_2),Y(\pi_2(X,\Xi_2)))\ge \nu_2(X,\Xi_1)
\quad\Longleftrightarrow\quad
u(\pi_1(X,\Xi),Y(\pi_1(X,\Xi)))\ge \nu_1(X,\Xi).
\$
Therefore $(\nu_2,\pi_2,\Xi_1,\Xi_2)$ is a feasible realization for $P_2$ and attains the same objective value as
$(\nu_1,\pi_1,\Xi)$ in $P_1$. Taking suprema over all feasible realizations $(\nu_1,\pi_1,\Xi)$ yields
$\OPT(P_2)\ge \OPT(P_1)$.

Combining the two inequalities gives $\OPT(P_1)=\OPT(P_2)$. And through the proof, we can construct an optimal realization of $P_1$ from an optimal realization of $P_2$ and construct an optimal realization of $P_2$ from an optimal realization of $P_1$.

\subsection{Proof of Proposition \ref{thm:non-randomized}}
\label{app:subsec_non-randomized}

\begin{proof}[Proof of Proposition \ref{thm:non-randomized}]
Let $(\nu,\pi)$ be any optimal solution to \ref{eq:RA-DPO-one-seed}.
Define, for each $x\in\cX$, $a\in\cA$, and $t\in\RR$,
\$
p(x,a,t):=\PP\big(u(a,Y(a))\ge t \given X=x\big).
\$

\paragraph{Step 1: construct a non-randomized selector on each $(x,t)$.}
For each $(x,t)$, let
\$
a^*(x,t)\ \in\ \argmax_{a\in\cA} p(x,a,t),
\$
and break ties deterministically by choosing the smallest index.
Then $a^*(x,t)$ is a deterministic map from $\cX\times\RR$ to $\cA$.

\paragraph{Step 2: define a modified policy that is constant on $\nu$-level sets.}
Set
\$
\tilde\pi(x,\xi)\ :=\ a^*\big(x,\nu(x,\xi)\big),\qquad (x,\xi)\in\cX\times[0,1].
\$
By construction, if $\nu(x,\xi)=\nu(x,\xi')$, then
\$
\tilde\pi(x,\xi)=a^*\big(x,\nu(x,\xi)\big)=a^*\big(x,\nu(x,\xi')\big)=\tilde\pi(x,\xi'),
\$
so $\tilde\pi$ satisfies the required non-randomization property.

\paragraph{Step 3: show feasibility is preserved (coverage constraint not decreased).}
Consider the coverage probability under $(\nu,\pi)$:
\$
\PP\big(u(\pi(X,\Xi),Y(\pi(X,\Xi)))\ge \nu(X,\Xi)\big)
=\EE\Big[\PP\big(u(\pi(X,\Xi),Y(\pi(X,\Xi)))\ge \nu(X,\Xi)\given X,\Xi\big)\Big].
\$
Fix $(X,\Xi)=(x,\xi)$ and write $t=\nu(x,\xi)$ and $a=\pi(x,\xi)$.
Using the assumption $\Xi\perp (X,\{Y(a)\}_{a\in\cA})$, we have
\$
\PP\big(u(\pi(X,\Xi),Y(\pi(X,\Xi)))\ge \nu(X,\Xi)\given X=x,\Xi=\xi\big)
=\PP\big(u(a,Y(a))\ge t \given X=x\big)
= p(x,a,t).
\$
Hence,
\$
\PP\big(u(\pi(X,\Xi),Y(\pi(X,\Xi)))\ge \nu(X,\Xi)\big)
=\EE\big[p\big(X,\pi(X,\Xi),\nu(X,\Xi)\big)\big].
\$
Similarly, for $\tilde\pi$ we obtain
\$
\PP\big(u(\tilde\pi(X,\Xi),Y(\tilde\pi(X,\Xi)))\ge \nu(X,\Xi)\big)
=\EE\big[p\big(X,\tilde\pi(X,\Xi),\nu(X,\Xi)\big)\big].
\$
But by definition of $\tilde\pi(X,\Xi)=a^*(X,\nu(X,\Xi))$,
\$
p\big(X,\tilde\pi(X,\Xi),\nu(X,\Xi)\big)
=\max_{a\in\cA} p\big(X,a,\nu(X,\Xi)\big)
\ \ge\ p\big(X,\pi(X,\Xi),\nu(X,\Xi)\big)
\quad \text{a.s.}
\$
Taking expectations yields
\$
\PP\big(u(\tilde\pi(X,\Xi),Y(\tilde\pi(X,\Xi)))\ge \nu(X,\Xi)\big)
\ \ge\
\PP\big(u(\pi(X,\Xi),Y(\pi(X,\Xi)))\ge \nu(X,\Xi)\big)
\ \ge\ 1-\alpha,
\$
so $(\nu,\tilde\pi)$ is feasible whenever $(\nu,\pi)$ is feasible.

\paragraph{Step 4: show optimality is preserved.}
The objective in \eqref{eq:RA-DPO-one-seed} depends only on $\nu$:
\$
\EE[\nu(X,\Xi)].
\$
We have not changed $\nu$, so the objective value of $(\nu,\tilde\pi)$ equals that of $(\nu,\pi)$.
Since $(\nu,\pi)$ is optimal and $(\nu,\tilde\pi)$ is feasible with the same objective value,
$(\nu,\tilde\pi)$ is also optimal. 
\end{proof}

\subsection{Proof of Theorem \ref{thm:continuous-cov}}
\label{app:subsec_cont_conv}

\begin{proof}[Proof of Theorem \ref{thm:continuous-cov}]
We first prove the results for $\hat{C}^{\text{full}}(x,a)$ and $\beta^{\text{full}}(y)$ in Remark~\ref{rem:orc_cal}. 
\paragraph{Step 1: Risk-averse policy $\pi_\RA(X_{\test};\hat{C}^{\text{full}})=\hat a(X_{\test})$.}
Recall that $\beta^{\text{full}}_0 \ge \beta^{\text{full}}(y)$ for all $y\in\cY$. 
Since $\hat g(x,\beta)$ is non-decreasing in $\beta$, we have
\$
\hat g(X_{\test},\beta^{\text{full}}_0) \ge \hat g(X_{\test},\beta^{\text{full}}(y)),\qquad \forall y\in\cY.
\$
Moreover, by the monotonicity of $\hat\theta(x,\cdot)$ (non-increasing in its second argument), it follows that
\$
\hat\theta\big(X_{\test},\hat g(X_{\test},\beta^{\text{full}}(y))\big)
\ge
\hat\theta\big(X_{\test},\hat g(X_{\test},\beta^{\text{full}}_0)\big),\qquad \forall y\in\cY.
\$
By the definition of $\hat\theta(\cdot,\cdot)$, we have  
\$
\hat\theta\big(X_{\test},\hat g(X_{\test},\beta^{\text{full}}_0)\big)\ge \hat\gamma\big(X_{\test},\hat g(X_{\test},\beta^{\text{full}}_0),a\big),\qquad \forall a\in\cA
\$
Therefore, for any $a \neq \hat a(X_{\test})$,
\$
\inf_{y \in \hat{C}^{\text{full}}(X_{\test},\hat a(X_{\test}))} u(\hat a(X_{\test}),y)
&\ge \inf_{y \in \cY} \hat\theta\big(X_{\test},\hat g(X_{\test},\beta^{\text{full}}(y))\big) \\
&= \hat\theta\big(X_{\test},\hat g(X_{\test},\beta^{\text{full}}_0)\big) \\
&\ge \hat\gamma\big(X_{\test},\hat g(X_{\test},\beta^{\text{full}}_0),a\big) \\
&= \inf_{y \in \hat{C}^{\text{full}}(X_{\test},a)} u(a,y),
\$
which implies $\pi_{\RA}(X_{\test};\hat{C}^{\text{full}})=\hat a(X_{\test})$ by definition. 

In the next, we prove the last equality: $\hat\gamma(X_{\test},\hat g(X_{\test},\beta^{\text{full}}_0),a)=\inf_{y \in \hat{C}^{\text{full}}(X_{\test},a)} u(a,y)$.
\begin{itemize}
    
\item First, if $\hat g(X_{\test},\beta^{\text{full}}_0)=0$, then $\hat\gamma(X_{\test},\hat g(X_{\test},\beta^{\text{full}}_0),a)=\inf_{y \in \hat{C}^{\text{full}}(X_{\test},a)} u(a,y)=u_{\max}$ by definition.
\item 
It remains to consider the case $\hat g(X_{\test},\beta^{\text{full}}_0)>0$. Let $m=\inf_{y\in\hat{C}^{\text{full}}(X_{\test},a)}u(a,y)$ and $q=\hat\gamma(X_{\test},\hat g(X_{\test},\beta^{\text{full}}_0),a)=\uquant_{1-\hat g(X_{\test},\beta^{\text{full}}_0)}[u(a,\widehat Y(a)) \given X=X_{\test}]$. Here, $\widehat Y(a)$ denotes a random variable generated from the fitted outcome model 
$\widehat P(\cdot\given X=X_{\test},A=a)$.
By the definition of $\hat{C}^{\text{full}}$, it is straightforward to see that $m\geq q$.
Let $\epsilon=\frac{m-q}{2}$. If $m>q$, then  there does not exist any $y\in\cY$ satisfying $u(a,y) \in [q,q+\epsilon]$. In particular, 
\$
\PP(u(a,\widehat Y(a)) \le q+\epsilon\given X=X_{\test})&=\PP(u(a,\widehat Y(a)) < q\given X=X_{\test})\\&=\lim_{s \uparrow q}\PP(u(a,\widehat Y(a)) \le s\given X=X_{\test}).
\$
By construction, we have $q=\uquant_{1-\hat g(X_{\test},\beta^{\text{full}}_0)}[u(a,\widehat Y(a)) \given X=X_{\test}]=\sup \{z \in \RR \given \PP(u(a,\widehat Y(a)) \le z \given X=X_{\test}) \le 1-\hat g(X_{\test},\beta^{\text{full}}_0)\}$. 
This implies  $\PP(u(a,\widehat Y(a))\le s \given X=X_{\test}) \le 1-\hat g(X_{\test},\beta^{\text{full}}_0)$ for any $s<q$. 
Therefore 
\$
\PP(u(a,\widehat Y(a)) \le q+\epsilon\given X=X_{\test})=\lim_{s\uparrow q}\PP(u(a,\widehat Y(a)) \le s\given X=X_{\test})\le 1-\hat g(X_{\test},\beta^{\text{full}}_0),
\$
which implies $\uquant_{1-\hat g(X_{\test},\beta^{\text{full}}_0)}[u(a,\widehat Y(a)) \given X=X_{\test}]\ge q+\epsilon$, a contradiction. Hence we must have $m=q$. 
\end{itemize}

\paragraph{Step 2: Coverage guarantee.} Since $\hat\pi_{\RA}(\cdot;\hat{C}^{\text{full}}) = \hat{a}(\cdot)$ is obtained independent of the data in $\cI_{\calib}$ and $X_{\test}$, we view it as fixed hereafter. 
For notational simplicity, relabel $\{(X_i,Y_i,w_i):i\in\cI_{\calib}^0\}$ as $(X_1,Y_1,w_1),\dots,(X_n,Y_n,w_n)$, and set $X_{n+1}=X_{\test}$, $Y_{n+1}=Y_{\test}(\hat a(X_{\test}))$, and $w_{n+1}=w_{\test}$. When the candidate label equals the true test label, abbreviate $S_{n+1}(\beta):=S_{n+1}^{(Y_{n+1})}(\beta)$.
Define the normalized weights
\$
p_i \coloneqq \frac{w_i}{\sum_{j=1}^{n+1} w_j},\qquad i=1,\dots,n+1,
\$
and the conformity scores
\$
V_i \coloneqq 1-\ind\{S_i(\beta^{\text{full}}(Y_{n+1}))\ge 0\},\qquad i=1,\dots,n+1.
\$
First, under the known behavior policy $\pi$, conditional on the membership in $\cI_\calib^0$, the samples $\{(X_i,Y_i)\}_{i\in\cI_{\calib}^0}$ are from the conditional distribution $\PP(X,Y(\hat{a}(X))\given A=\hat{a}(X))$. 
By Bayes' rule, the density ratio between $(X_{\test},Y(\hat{a}(X_{\test})))$ and these calibration samples is given by 
\$
\frac{\ud \PP_{X_{\test},Y(\hat{a}(X_{\test}))}}{\ud \PP_{X,Y(\hat{a}(X))\given A=\hat{a}(X)}} (x,y) \propto \frac{1}{\PP(A=\hat{a}(X)\given X=x,Y(\hat{a}(X))=y)}  = \frac{1}{ \PP(A=\hat{a}(X)\given X=x )} = \frac{1}{\pi(\hat{a}(x)\given x)}.
\$
That is, the samples $\{(X_i,Y_i(\hat{a}(X_i)))\}_{i=1}^{n+1}$ are subject to a covariate shift with density ratio $1/\pi(\hat{a}(x)\given x)$. Therefore, they obey weighted exchangeability with weights given by $\{p_i\}_{i=1}^{n+1}$~\citep{tibshirani2019conformal}. 

We now proceed to show the desired coverage guarantee. By the definition of $\beta^{\text{full}}(y)$ in~\eqref{eq:beta_star}, we have
\$
\frac{\sum_{i=1}^{n+1} w_i\,\ind\{S_i(\beta^{\text{full}}(Y_{n+1}))\ge 0\}}
{\sum_{i=1}^{n+1} w_i}
\ge 1-\alpha, \quad \Longleftrightarrow \quad \quant\Big(1-\alpha;\sum_{i=1}^{n+1} p_i\,\delta_{V_i}\Big) = 0,
\$
where Quantile($\beta;P$)$=\inf\{z\colon P(Z\leq z)\geq \beta\}$ is the quantile of a distribution $P$, and the ``='' is because $V_i\in\{0,1\}$. 
Consequently, by construction in~\eqref{eq:def_cp_set}, 
\begin{align*}
\PP\big(Y_{n+1}\in \hat C^{\text{full}}(X_{n+1},\hat a(X_{n+1}))\big)
&= \PP(V_{n+1}=0) \\
&= \PP\left(V_{n+1}\le 
\quant\Big(1-\alpha;\sum_{i=1}^{n+1} p_i\,\delta_{V_i}\Big)\right) \\
&= \PP\left(V_{n+1}\le 
\quant\Big(1-\alpha;\sum_{i=1}^{n} p_i\,\delta_{V_i}+p_{n+1}\delta_{+\infty}\Big)\right) \\
&\ge 1-\alpha,
\end{align*}
where the last inequality follows from Theorem 2 in \cite{tibshirani2019conformal}. We thus complete the proof of the first part of Theorem~\ref{thm:continuous-cov}. 

We now proceed to prove the second part of Theorem~\ref{thm:continuous-cov}, i.e., the coverage guarantee of $\hat{C}$ in~\eqref{eq:est_pred_set}. Again, we first derive the risk-averse policy and then prove the coverage guarantee.

\paragraph{Step 3: Risk-averse policy  $\pi_{\RA}(X_{\test};\hat{C}) = \hat{a}(X_{\test})$.} We show that the risk-averse policy under $\hat{C}$ in~\eqref{eq:est_pred_set} also coincides with $\hat{a}(\cdot)$. 
Note that in~\eqref{eq:est_pred_set}, we have $\hat\theta(X_{\test},\hat{g}(X_{\test},\beta^*)) \geq \hat\gamma(X_{\test},\hat{g}(X_{\test},\beta^*),a)$ for any $a\neq \hat{a}(X_{\test})$ by the definition of $\hat\theta(\cdot,\cdot)$. Therefore 
\$
\inf_{y \in \hat{C}(X_{\test},\hat a(X_{\test}))} u(\hat a(X_{\test}),y)
&\ge \hat\theta\big(X_{\test},\hat g(X_{\test},\beta^*)\big) \\
&\ge \hat\gamma\big(X_{\test},\hat g(X_{\test},\beta^*),a\big) \\
&= \inf_{y \in \hat{C}(X_{\test},a)} u(a,y),
\$
where the last equality follows by the same argument in Step 1.
This implies the worst-case utility among $\hat{C}(X_{\test},a)$ must be attained by $a = \hat{a}(X_{\test})$, thus $\pi_{\RA}(X_{\test};\hat{C}) = \hat{a}(X_{\test})$.

\paragraph{Step 4: Coverage guarantee via $\hat{C}(x,a)\supseteq \hat{C}^{\text{full}}(x,a)$.} 

For any fixed $y\in\cY$, let $\beta^{\text{full}}(y)$ be defined as in \eqref{eq:beta_star_y}. Since
\$
\frac{\sum_{i\in\cI_{calib}^0} w_i\,\ind\{S_i(\beta^{\text{full}}(y))\ge 0\} + w_{\test}\ind\{S_{\test}^{(y)}(\beta^{\text{full}}(y))\ge 0\}}
{\sum_{i\in\cI_{calib}^0} w_i + w_{\test}}
\ge
\frac{\sum_{i\in\cI_{calib}^0} w_i\,\ind\{S_i(\beta^{\text{full}}(y))\ge 0\}}
{\sum_{i\in\cI_{calib}^0} w_i + w_{\test}},
\$
we must have $\beta^*\ge \beta^{\text{full}}(y)$ for all $y\in\cY$ by the definition of $\beta^{\text{full}}(y)$.
Since $\hat g(x,\beta)$ is non-decreasing in $\beta$ for any $x\in \cX$, we know 
\$
\hat g(X_{\test},\beta^*)\ge \hat g(X_{\test},\beta^{\text{full}}(y)),\qquad \forall ~ y\in\cY.
\$
By the monotonicity  (non-increasing) of $\hat\theta(x,t )$ in the argument $t$, and the monotonicity (non-increasing) of $\hat\gamma(x,t,\beta)$  in the argument $t$, this implies
\$
\hat\theta\big(X_{\test},\hat g(X_{\test},\beta^*)\big)
\le
\hat\theta\big(X_{\test},\hat g(X_{\test},\beta^{\text{full}}(y))\big), 
\quad \\ 
\text{and}\quad 
\hat\gamma\big(X_{\test},\hat g(X_{\test},\beta^*),a\big)
\le
\hat\gamma\big(X_{\test},\hat g(X_{\test},\beta^{\text{full}}(y)),a\big).
\$
Therefore, we know $\hat{C}^{\text{full}}(X_{\test},\hat{a}(X_{\test}) )\subseteq \hat{C}(X_{\test}, \hat{a}(X_{\test}))$. Since the risk-averse policies of the two suites of prediction sets coincide, we obtain the desired coverage guarantee for $\hat{C}(x,a)$ as well. This completes the proof of Theorem~\ref{thm:continuous-cov}. 
\end{proof}

\section{Numerical experiment details}
\subsection{Simulation data-generating process}
\label{app:simulation-dgp}

We generate data from a linear-softmax environment with covariate space $\mathcal X = \RR^d$, action space $\cA = \{0,1,2\}$, and label space $\cY = \{0,1,2,3\}$. In our experiments, we use $d=10$.

First, the covariate vector is sampled as
\$
X \sim \cN(0, I_d).
\$

Next, the logged action $A$ is drawn from a behavior policy $\pi_b(a \given X)$ defined by a multinomial softmax model. For each action $a \in \cA$, let
\$
\ell_a(X) = \frac{v_a^\top X + c_a}{\sqrt d},
\$
where the behavior-policy parameters are sampled independently as
\$
v_a \stackrel{\text{i.i.d.}}{\sim} \cN(0, I_d),
\qquad
c_a \stackrel{\text{i.i.d.}}{\sim} \cN(0, 0.5^2).
\$
Then
\$
\pi_b(a \given X)
=
\frac{\exp(\ell_a(X))}
{\sum_{a' \in \cA} \exp(\ell_{a'}(X))},
\qquad
A \given X \sim \pi_b(\cdot \given X).
\$

Finally, conditional on $(X,A)$, the outcome $Y$ is generated from an action-dependent multinomial softmax model. For each action $a \in \cA$ and label $y \in \cY$, define
\$
r_{a,y}(X) = 1.2 \bigl(w_{a,y}^\top X + b_{a,y}\bigr),
\$
where the outcome-model parameters are sampled independently as
\$
w_{a,y} \stackrel{\text{i.i.d.}}{\sim} \cN\!\left(0, \frac{1}{d} I_d\right),
\qquad
b_{a,y} \stackrel{\text{i.i.d.}}{\sim} \cN(0, 0.5^2).
\$
The conditional outcome distribution is
\$
\PP(Y = y \given X, A=a)
=
\frac{\exp(r_{a,y}(X))}
{\sum_{y' \in \cY} \exp(r_{a,y'}(X))},
\qquad
Y \given X,A \sim \PP(\cdot \given X,A).
\$

For each simulation replicate, all behavior-policy and outcome-model parameters are sampled once using the replicate-specific random seed and then held fixed for all observations within that replicate.

\subsection{Plug-in baseline}
\label{app:plugin}

We describe the plug-in baseline used in the simulation. Unlike \textsc{PC-RACP}, this baseline does not include a learning step for a global coverage-allocation parameter and does not perform any conformal calibration for finite-sample exact coverage. Instead, it directly uses the fitted action-conditional outcome model $\widehat P(Y \given X,A)$ to choose an action and construct a prediction set based on the optimal formula in Theorem~\ref{thm:optimal-set}.

Fix a target miscoverage level $\alpha \in (0,1)$. After fitting the outcome model $\widehat P(\cdot \given x,a)$ for the distribution of $Y(a)$ conditional on $X=x$, the method computes the utility threshold:  
\$
\hat \gamma(x,1-\alpha,a)
:=
\uquant_{\alpha}\big[u(a,\widehat Y(a))\given \widehat Y(a)\sim \widehat{P}(\cdot\given X=x,A=a)\big].
\$

The baseline then chooses the action with the largest plug-in utility quantile:
\$
\hat a_{\mathrm{plug}}(x,1-\alpha)
\in
\argmax_{a \in \cA} \hat \gamma(x,1-\alpha,a),
\$
and sets
\$
\hat \theta_{\mathrm{plug}}(x,1-\alpha)
:=
\max_{a \in \cA} \hat \gamma(x,1-\alpha,a).
\$
The resulting plug-in prediction set is
\$
\widehat C_{\mathrm{plug}}(x,a)
:=
\bigl\{y \in \cY : u(a,y) \ge \hat \gamma(x,1-\alpha,a)\bigr\}.
\$
By the same argument as in Appendix~\ref{app:subsec_cont_conv}, the associated risk-averse policy is $\pi_\RA(X;\widehat C_{\mathrm{plug}})=\hat a_{\mathrm{plug}}(X,1-\alpha)$.

\subsection{Evaluation in real data}
\label{app:subsec_eval_real}

Here we outline the consistent estimator for coverage and utility in the randomized experiment data in Section~\ref{sec:hillstrom-exp}. 
Suppose our policy is $\hat{a}(X)$ and associated selected-action prediction set is $\hat{C}(X)$. 
Due to randomization, we know $\PP_{X_{\text{test}},Y_{\test}(t) \given T_{\test}=t}=\PP_{X_{\text{test}},Y_{\test}(t)}$ for $t\in \{0,1\}$. 
The coverage can then be written as (viewing $\hat{a}$ and $\hat{C}$ as fixed)
\$
& \PP(Y_\test(\hat{a}(X_\test)) \in \hat{C}(X_\test)) \\ &= 
\EE\Big[ \ind\{\hat{a}(X_{\test})=1, Y_\test(1) \in \hat{C}(X_\test) \} + \ind\{\hat{a}(X_{\test})=0, Y_\test(0) \in \hat{C}(X_\test) \} \Big] \\ 
&= \EE\Big[ \ind\{\hat{a}(X_{\test})=1, Y_\test(1) \in \hat{C}(X_\test) \} \Biggiven T_\test = 1\Big] \\ 
&\quad + \EE\Big[ \ind\{\hat{a}(X_{\test})=0, Y_\test(0) \in \hat{C}(X_\test) \} \Biggiven T_\test = 0\Big]\\ 
&= \EE\Big[ \ind\{\hat{a}(X_{\test})=1, Y_\test  \in \hat{C}(X_\test) \} \Biggiven T_\test = 1\Big] \\ 
&\quad + \EE\Big[ \ind\{\hat{a}(X_{\test})=0, Y_\test  \in \hat{C}(X_\test) \} \Biggiven T_\test = 0\Big]
\$
The last two quantities can now be estimated by sample averages in actually treated and actually control units in the test samples. 

\subsection{CatBoost Hyperparameters for the Hillstrom Experiment}
For the Hillstrom experiment, we use CatBoost classifiers to estimate the conditional outcome distribution $\widehat P(Y\given X,A)$ and the marginal label distribution $\widehat P(Y\given X)$. The conditional outcome model $\widehat P(Y\given X,A)$ is fitted separately within each action group, using one CatBoost classifier for each value of $A$. The marginal label model $\widehat P(Y\given X)$ is fitted once on the pooled training data. All CatBoost classifiers are implemented using \texttt{CatBoostClassifier} from the Python library \texttt{catboost}.

The CatBoost classifiers use the following hyperparameters:
\$
\texttt{loss\_function}=\texttt{Logloss},\quad \texttt{eval\_metric}=\texttt{Logloss},\quad \texttt{iterations}=500, \\
\texttt{depth}=6,\quad \texttt{learning\_rate}=0.03,\quad \texttt{random\_seed}=0,
\$
with \texttt{verbose=False} and \texttt{allow\_writing\_files=False}. No additional hyperparameter tuning is performed in the reported experiments. For the Hillstrom CatBoost experiment, we set the parameter \texttt{standardize=False}, since CatBoost can handle mixed numeric and categorical covariates directly and does not require feature standardization.

\subsection{Random Forest Hyperparameters for the Simulation Experiments}
In the experiments using Random Forests, we implement the base probabilistic classifiers with \texttt{RandomForestClassifier} from \texttt{scikit-learn}. For the simulation experiments, Random Forests are used to estimate the conditional outcome model, the behavior policy model, and the marginal label model.

For the simulation experiments, the Random Forest classifiers use the following hyperparameters:
\$
\texttt{n\_estimators}=200,\quad \texttt{max\_depth}=14,\quad \texttt{min\_samples\_leaf}=2,\quad \texttt{n\_jobs}=1.
\$
The random seed of each Random Forest model is matched to the seed of the corresponding experimental run.

All Random Forest hyperparameters not explicitly specified, including the split criterion, feature subsampling rule, bootstrap sampling, and the minimum number of samples required for an internal split, are kept at their default \texttt{scikit-learn} values. No additional hyperparameter tuning is performed within the simulation scripts.
\end{document}